\documentclass[letterpaper, 10 pt, conference]{ieeeconf}  
\IEEEoverridecommandlockouts 
\usepackage{amsmath,amssymb,color,cite}
\usepackage{graphicx}

\usepackage{latexsym}
\usepackage{algorithm}
\usepackage{algpseudocode}
\usepackage{verbatim}
\usepackage{cite}
\usepackage[english]{babel}
\usepackage{balance}

\usepackage{amsmath,amsthm}

\let\labelindent\relax
\usepackage{enumitem}
\usepackage{xcolor}

\usepackage{pgf}
\usepackage{pgfplots,tikz}
\usetikzlibrary{tikzmark}
\usepackage[font=footnotesize]{caption}
\usepackage{arydshln}

\usepackage{amsfonts}
\usepackage{amssymb}
\usepackage{amsbsy}
\usepackage{bm}
\usepackage{mathtools}
\usepackage{hyperref}
\usepackage[font=footnotesize]{subcaption}
\usepackage[font=footnotesize]{caption}
\usepackage{yfonts}

\theoremstyle{definition}
\newtheorem{definition}{Definition}

\newtheorem{remark}{Remark}
\theoremstyle{plain}
\newtheorem{assumption}{Assumption}
\newtheorem{lemma}{Lemma}
\newtheorem{theorem}{Theorem}
\newtheorem{proposition}{Proposition}
\newtheorem{corollary}{Corollary}

\newcommand{\barf}{\overline{f}}
\newcommand{\barg}{\overline{g}}
\newcommand{\xstar}{x^*}
\newcommand{\ystar}{y^*}
\newcommand{\Qstar}{Q^*}
\newcommand{\rhostar}{\rho^*}
\newcommand{\Scal}{\mathcal{S}}
\newcommand{\UU}{\mathcal{U}}
\newcommand{\sfrak}{\mathfrak{s}}
\newcommand{\ufrak}{\mathfrak{u}}
\newcommand{\hA}{\hat{A}}
\newcommand{\hB}{\hat{B}}
\newcommand{\hC}{\hat{C}}
\newcommand{\hD}{\hat{D}}
\newcommand{\RR}{\mathbb{R}}
\newcommand{\EE}{\mathbb{E}}
\newcommand{\PP}{\mathbb{P}}
\newcommand{\FF}{\mathcal{F}}
\newcommand{\Acal}{\mathcal{A}}
\newcommand{\Bcal}{\mathcal{B}}
\newcommand{\xx}{\mathbf{x}}

\newcommand{\cc}{\mathfrak{c}}
\newcommand{\afrak}{\mathfrak{a}}

\allowdisplaybreaks

\title{\LARGE \bf
Finite-Time Bounds for Two-Time-Scale Stochastic Approximation with Arbitrary Norm Contractions and Markovian Noise
}

\author{Siddharth Chandak$^{1\dagger}$, Shaan Ul Haque$^{2\dagger}$ and Nicholas Bambos$^{3}$
\thanks{$\dagger$ Equal contribution. Listed alphabetically.}
\thanks{$^{1,3}$ Department of Electrical Engineering, Stanford University, USA}
\thanks{$^{2}$ {Industrial and Systems Engineering, Georgia Institute of Technology, Atlanta, USA}}
\thanks{$^{1}${\tt\small chandaks@stanford.edu}}
\thanks{$^{2}${\tt\small shaque49@gatech.edu}}
\thanks{$^{3}${\tt\small bambos@stanford.edu}}  
}

\begin{document}

\maketitle
\thispagestyle{empty}
\pagestyle{empty}

\begin{abstract}
Two-time-scale Stochastic Approximation (SA) is an iterative algorithm with applications in reinforcement learning and optimization. Prior finite-time analyses of such algorithms have focused on fixed point iterations with mappings contractive under the Euclidean norm. Motivated by applications in reinforcement learning, we give the first mean square bound for non-linear two-time-scale SA with arbitrary norm contractive mappings and Markovian noise. We show that the mean square error decays at a rate of $\mathcal{O}(1/n^{2/3})$ in the general case, and at a rate of $\mathcal{O}(1/n)$ in a special case where the slower timescale is noiseless. Our analysis uses the generalized Moreau envelope to handle the arbitrary norm contractions and solutions of Poisson equation to deal with the Markovian noise. By analyzing the SSP Q-Learning algorithm, we give the first $\mathcal{O}(1/n)$ bound for an algorithm for asynchronous control of MDPs under the average reward criterion. We also obtain a rate of $\mathcal{O}(1/n)$ for Q-Learning with Polyak-averaging as a special case of our result and provide an algorithm for learning Generalized Nash Equilibrium (GNE) for strongly monotone games which converges at a rate of $\mathcal{O}(1/n^{2/3})$.
\end{abstract}
\section{Introduction}
Stochastic Approximation (SA) is a popular class of iterative algorithms for finding the fixed point of an operator given its noisy realizations. Many Reinforcement Learning (RL) algorithms \cite{sutton2018} involve solving fixed point equations under martingale and Markovian noise, and hence can be modeled using SA. Beyond RL, SA algorithms have applications in optimization, communications and control \cite{Borkar-book}. 

In many scenarios, including RL and optimization, problem structure requires updating iterates at different rates. Some of the well-known examples include GTD and TDC methods for policy evaluation with offline data \cite{sutton2018}, SSP Q-Learning for control problem in average-reward RL \cite{abounadi}, SA with Polyak-Ruppert Averaging \cite{polyak1992acceleration} or solving Generalized Nash Equilibrium Problems \cite{GNEP}. Broadly, these algorithms are instances of two-time-scale SA which we study in this work. Specifically, we study the following iterations:
\begin{equation*}
    \begin{aligned}
&x_{n+1}=x_n+\alpha_n (f(x_n,y_n,Z_n)-x_n+M_{n+1})\\
&y_{n+1}=y_n+\beta_n (g(x_n,y_n,Z_n)-y_n+M'_{n+1}).
\end{aligned}
\end{equation*}
Here $x_n\in\RR^{d_1}$ and $y_n\in\RR^{d_2}$ are the iterates updated on the faster and slower time-scale, respectively. These time-scales are dictated by the stepsizes $\alpha_n$ and $\beta_n$, respectively, where $\beta_n$ decays faster than $\alpha_n$. $\{Z_n\}$ is the underlying Markov chain which leads to the Markovian noise and $M_{n+1}$ and $M'_{n+1}$ denote additive martingale noise. The objective of the algorithm is to find solutions to $\barf(\xstar,\ystar)=\xstar$ and $\barg(\xstar,\ystar)=\ystar$, where $\barf(x,y)$ and $\barg(x,y)$ are averages of $f(x,y,i)$ and $g(x,y,i)$, respectively, under the stationary distribution $\pi(i)$ for the Markov chain $\{Z_n\}$.

While asymptotic convergence of SA has been widely studied in SA literature, applications in RL have led to recent interest in obtaining finite-time guarantees on SA algorithms. More specifically, works obtaining expectation bounds on SA algorithms predominantly assume contractions (or some negative drift) under the Euclidean norm. However, many RL algorithms exhibit the contractive property under different norms such as the $\ell_\infty$ norm. A recent work \cite{chenfinite} used generalized Moreau envelopes to obtain expectation bound for single-time-scale RL algorithms which involve a contraction mapping under arbitrary norms. We extend their result to two-time-scale SA, obtaining the first such finite time bound for RL algorithms such as SSP Q-learning for average cost Markov Decision Processes (MDPs) \cite{abounadi}.

Finite-time analysis for two-time-scale algorithms often relies on the assumption that that mappings $\barf(x,y)$ and $\barg(\xstar(y),y)$ are contractive with respect to $x$ and $y$, respectively. This is sufficient to ensure uniqueness of $(\xstar,\ystar)$ and asymptotic convergence of the algorithm for appropriate choice of stepsizes. Here $\xstar(y)$ denotes the fixed point of $\barf(\cdot,y)$. While prior works have assumed these contractions under Euclidean norm \cite{Doan}, we assume arbitrary norm contractions. To deal with arbitrary norms, we construct smooth Lyapunov functions using the generalized Moreau envelope. In addition, we use solutions of Poisson equation to deal with the Markovian noise. Our mean square bounds require novel analysis to incorporate the above two techniques in the analysis for two-time-scale iterations. 

Our main contributions are as follows.
\begin{enumerate}
    \item We show that $\EE[\|x_n-\xstar(y_n)\|^2+\|y_n-\ystar\|^2]$ decays as $\mathcal{O}(1/n^{2/3})$. This is the first such bound for two-time-scale SA with arbitrary norm contractions. We also show almost sure convergence of iterates to $(\xstar,\ystar)$.
    \item We also show that the above mean square error decays at $\mathcal{O}(1/n)$ when the slower time-scale is noiseless, i.e., lacks both martingale and Markovian noise.
    \item We show that SSP Q-learning for average cost MDPs satisfies the second framework above, and give the first finite time bound of $\mathcal{O}(1/n)$ for the algorithm. This is the first $\mathcal{O}(1/n)$ bound for an asynchronous algorithm for average cost MDPs. We also present the first a.s.\ convergence result for this algorithm which does not require projection of iterates in the slower time-scale.
    \item We adapt Q-learning with Polyak averaging into our framework and obtain a rate of $\mathcal{O}(1/n)$ for stepsizes independent of the system parameters.
    \item We also present a two-time-scale algorithm for learning GNE in strongly monotone games with linear coupled constraints, achieving a convergence rate of $\mathcal{O}(1/n^{2/3})$.
\end{enumerate}

\subsection{Related Work}
There is a vast literature on SA which focuses on asymptotic analysis of the algorithms (see \cite{Kushner, Borkar-book} for textbook references). Due to applicability in RL and optimization, there has been a recent interest in bounding the finite-time performance of these algorithms. These works can be further divided into two categories: obtaining high probability bounds on these algorithms (e.g., see \cite{Chandak_conc, chen-conc} and references therein) and obtaining mean square bounds \cite{srikant2019finite, qu2020finite, haque2024stochastic, chen2024lyapunov, zhang2024constant}. In particular, \cite{chenfinite, chen2024lyapunov} give mean square bounds for RL algorithms with contractions under arbitrary norms in the martingale and Markov noise cases, respectively. 

Among two-time-scale SA, there has been special focus on obtaining bounds for linear two-time-scale SA, where the maps $f(\cdot)$ and $g(\cdot)$ are linear. This is due to linear two-time-scale RL algorithms such as TDC and GTD2 \cite{tdc}. In \cite{konda}, an asymptotic rate of $\mathcal{O}(1/n)$ is obtained, while in \cite{Shaan, Kaledin} this has been extended to finite-sample guarantees. In \cite{Dalal} high probability guarantees have been given for such iterations. Note that our results can also be applied on linear two-time-scale SA, but we obtain weaker bounds as we do not exploit the additional structure provided by linearity.

Bounds for non-linear contractive two-time-scale SA have also been recently studied in  \cite{Doan}, where they obtained a bound of $\mathcal{O}(1/n^{2/3})$ for the Euclidean norm contractions case. We extend their result to arbitrary norms and we further include Markov noise and martingale noise which scales linearly in norm with the iterates. These assumptions allow us to capture a broader range of algorithms in our framework. 

In recent years, two-time-scale SA has also been studied in cases where the contractive assumption is not satisfied. While the mapping in the faster time-scale is often assumed to be contractive, mappings in the slower time-scale are assumed to be non-expansive in \cite{Chandak-nonexp} and the slower time-scale is a gradient descent iteration for a smooth operator in the limit of $x\rightarrow\xstar(y)$ in \cite{Doan-grad}. In \cite{Borkar-alekseev}, no contractive assumption is made and the rate of convergence of iterates to the solutions of the corresponding ODE is studied.

\subsection{Notation}

Throughout this work, we work with vectors of dimensions $d_1$ and $d_2$ as iterates $x_n\in\RR^{d_1}$ and $y_n\in\RR^{d_2}$. For simplicity, we use $\|\cdot\|$ to denote the norms in both $\RR^{d_1}$ and $\RR^{d_2}$, i.e., $\|x\|$ and $\|y\|$ denote compatible norms in $\RR^{d_1}$ and $\RR^{d_2}$, respectively, when $x\in\RR^{d_1}$ and $y\in\RR^{d_2}$. The usage will be clear from context. For $x_1,x_2\in\RR^d_1$, $\langle x_1,x_2\rangle_2=x_1^Tx_2$ denotes the standard inner product. Other inner products and norms are defined when used.

\section{Problem Formulation and Main Result}\label{sec:main}
In this section, we present the main result of this paper after setting up the necessary notation and assumptions. Recall the following coupled iterations.
\begin{equation}\label{iter-main}
    \begin{aligned}
&x_{n+1}=x_n+\alpha_n (f(x_n,y_n,Z_n)-x_n+M_{n+1})\\
&y_{n+1}=y_n+\beta_n (g(x_n,y_n,Z_n)-y_n+M'_{n+1}).
\end{aligned}
\end{equation}
We impose the following assumptions on these iterations, beginning with the required conditions on the Markov noise.
\begin{assumption}\label{assu-Markov}
    The random process $\{Z_n\}$ is a Markov chain taking values in a finite state space $\Scal$ and satisfies
    \begin{align*}
        \PP(Z_{n+1}=j\mid Z_m, m\leq n)&=\PP(Z_{n+1}=j\mid Z_n)\\
        &=p(j\mid Z_n),
    \end{align*}
    for $j\in\Scal$. $p(\cdot|\cdot)$ is the transition probability of an irreducible Markov chain with unique stationary distribution $\pi(\cdot)$.
\end{assumption}
We next define functions $\barf(\cdot):\RR^{d_1}\times\RR^{d_2}\mapsto\RR^{d_1}$ and $\barg(\cdot):\RR^{d_1}\times\RR^{d_2}\mapsto\RR^{d_2}$ as the expectation of functions $f(\cdot):\RR^{d_1}\times\RR^{d_2}\times\Scal\mapsto\RR^{d_1}$ and $g(\cdot):\RR^{d_1}\times\RR^{d_2}\times\Scal\mapsto\RR^{d_2}$ with respect to the distribution $\pi$ over the state space $\Scal$. 
$$\barf(x,y)=\sum_{i\in\Scal}\pi(i)f(x,y,i) \; \textrm{and} \; \barg(x,y)=\sum_{i\in\Scal}\pi(i)g(x,y,i).$$
Note that $f(x_n,y_n,Z_n)-\barf(x_n,y_n)$ and $g(x_n,y_n,Z_n)-\barg(x_n,y_n)$ act as the Markovian noise in our framework.

The algorithm's objective is to find fixed points for $\barf(\cdot)$ and $\barg(\cdot)$, i.e., find $(x,y)$ such that $\barf(x,y)=x$ and $\barg(x,y)=y$. We next assume that $\barf(x,y)$ is contractive in $x$. 
\begin{assumption}\label{assu-contractive-f}
Function $\barf(x,y)$ is $\lambda$-contractive in $x$, i.e., there exists a $\lambda\in [0,1)$ such that 
\begin{equation}
    \|\barf(x_1,y)-\barf(x_2,y)\|\leq \lambda \|x_1-x_2\|,
\end{equation}
for all $x_1,x_2\in\RR^{d_1}, y\in\RR^{d_2}$.
\end{assumption}
By the Banach contraction mapping theorem, the above assumption implies that $\barf(\cdot,y)$ has a unique fixed point for each $y$. That is, for each $y\in\RR^{d_2}$, there exists $\xstar(y)$ such that $\barf(\xstar(y),y)=\xstar(y)$. Our next assumption is that the function $\barg(\xstar(y),y)$ is contractive in $y$.
\begin{assumption}\label{assu-contractive-g}
Function $\barg(\xstar(\cdot),\cdot)$ is $\mu$-contractive, i.e., there exists a $\mu\in [0,1)$ such that for all $y_1,y_2\in\RR^{d_2}$.
\begin{equation}
    \|\barg(\xstar(y_1),y_1)-\barg(\xstar(y_2),y_2)\|\leq \mu \|y_1-y_2\|.
\end{equation}
\end{assumption}
The above assumption implies the unique existence of $\ystar$ such that $g(\xstar(\ystar),\ystar)=\ystar$. For simplicity, we further define $\xstar(\ystar)=\xstar$. We next present two assumptions, standard in analysis of SA schemes \cite{Borkar-book, Chandak_conc}. The first assumption is about the martingale difference sequences. 
\begin{assumption}\label{assu-Martingale}
    Define the family of $\sigma$-fields $\FF_n=\sigma(x_0,y_0,Z_k,M_k,M'_k, k\leq n)$. Then $\{M_{n+1}\}$ and $\{M'_{n+1}\}$ are martingale difference sequences with respect to $\FF_n$, i.e., $\EE[M_{n+1}\mid\FF_n]=\EE[M'_{n+1}\mid\FF_n]=0.$
Moreover, for $K>0$,
    $$\|M_{n+1}\|+\|M'_{n+1}\|\leq K(1+\|x_n\|+\|y_n\|),\;\forall n\geq 0.$$
\end{assumption}
We next assume that that the mappings are Lipschitz.

\begin{assumption}\label{assu-Lipschitz}
Functions $f(\cdot)$ and $g(\cdot)$ are $L$-Lipschitz, i.e.,
\begin{align*}
    &\|f(x_1,y_1,i)-f(x_2,y_2,i)\|\leq L(\|x_1-x_2\|+\|y_1-y_2\|)\\
    &\|g(x_1,y_1,i)-g(x_2,y_2,i)\|\leq L(\|x_1-x_2\|+\|y_1-y_2\|),
\end{align*}
for all $x_1,x_2\in\RR^{d_1}, y_1,y_2\in\RR^{d_2}$ and $i\in\Scal$. Moreover,
$$\|f(x,y,i)\|+\|g(x,y,i)\|\leq K(1+\|x\|+\|y\|),$$
for $x\in\RR^{d_1}, y\in\RR^{d_2},i\in\Scal$,
\end{assumption}

We finally make an assumption about the step-size sequences $\alpha_n$ and $\beta_n$.
\begin{assumption}\label{assu-stepsize}
The step size sequences are of the form $$\alpha_n=\frac{\alpha_0}{(n+1)^{\afrak}}\;\text{and}\;\beta_n=\frac{\beta_0}{n+1},$$
where $0<\alpha_0,\beta_0<1$ and $0.5<\afrak<1$. 
\end{assumption}
This assumption implies that the two sequences are non-increasing, and non-summable but square-summable. We also define the sequence $\gamma_n=\beta_n/\alpha_n$. This further implies that $\lim_{n\uparrow \infty}\gamma_n=0$. 
The above assumption also implies that for constants $\cc_1=2$ and $\cc_2=2/\min(\alpha_0, \beta_0)$ we have the following bounds for all $n\geq 0$,
$$\alpha_n\leq \cc_1\alpha_{n+1}\;\text{and}\;\beta_n\leq \cc_1\beta_{n+1},$$
$$\alpha_{n}-\alpha_{n+1}\leq \cc_2\alpha_{n+1}^2\;\text{and}\;\beta_{n}-\beta_{n+1}\leq \cc_2\beta_{n+1}^2.$$

\begin{theorem}\label{thm:expectation-bound}
    Suppose that Assumptions \ref{assu-Markov}-\ref{assu-stepsize} hold. Then the iterates $(x_n,y_n)$ almost surely converge to the fixed points $(\xstar,\ystar)$. There exists constant $C_1,C_2,C_3>0$ and time $n_0\geq 0$ such that for all $n\geq n_0$ and $\beta_0\geq C_3$,
    \begin{align*}
        &\EE\left[\|x_n-\xstar\|^2\right]\leq C_1\left(\alpha_n+\frac{\beta_n^2}{\alpha_n^2}\right),\\
    \text{and}\;\;&\EE\left[\|y_n-\ystar\|^2\right]\leq C_2\left(\alpha_n+\frac{\beta_n^2}{\alpha_n^2}\right).
    \end{align*}
\end{theorem}
For explicit values of $C_1,C_2$ and $C_3$, refer to the proof in Appendix I.

\subsection{Optimal Rate and Step Size Choice}\label{subsec:step_choice}
The best possible rate based on our analysis is obtained when $\alpha_n=\Theta(\beta_n^2/\alpha_n^2)$ or $\beta_n^2=\Theta(\alpha_n^3)$. This is achieved at $\afrak=2/3$, i.e., $\alpha_n=\alpha_0/(n+1)^{2/3}$ and $\beta_n=\beta_0/(n+1)$. This matches the rate obtained when the functions $\barf(\cdot)$ and $\barg(\cdot)$ are contractive mappings under the Euclidean norm \cite{Doan}, and hence we generalize the bound in \cite{Doan} to arbitrary norms.

Note that the time $n_0$ after which our bound holds is a fixed constant depending on the step sizes, contraction factors, and the norm $\|\cdot\|$. By taking assumptions similar to Assumption A2 from \cite{Kaledin}, our result can be shown to hold for all time $n\geq 0$.  

\section{An Important Special Case}\label{sec:special}
The formulation in Section \ref{sec:main} allows for martingale and Markov noises in both time-scales. While this allows us to cover a broader set of algorithms, the presence of the noise terms in the slower time-scale restricts the range of step size sequences our analysis can allow. As we discuss later in Section \ref{sec:applications}, both SSP Q-Learning and Q-Learning with Polyak averaging are noiseless in the slower time-scale. Hence, in this section, we consider this special case and give a bound of $\mathcal{O}(1/n)$ for such algorithms. Consider the following coupled iterations.
\begin{equation}\label{iter-special}
    \begin{gathered}
x_{n+1}=x_n+\alpha_n (f(x_n,y_n,Z_n)-x_n+M_{n+1})\\
y_{n+1}=y_n+\beta_n (\barg(x_n,y_n)-y_n).
\end{gathered}
\end{equation}
Note that this is a special case of $\eqref{iter-main}$ as we can define $M'_{n+1}=0$ for all $n$ and $g(x,y,i)=\barg(x,y)$ for all $x\in\RR^{d_1}, y\in\RR^{d_2}$ and $i\in\Scal$. The assumptions of contractivity (Assumption \ref{assu-contractive-f} and \ref{assu-contractive-g}) are still satisfied, i.e., $\barf(x,y)$ and $\barg(\xstar(y),y)$ are $\lambda$-contractive in $x$ and $\mu$-contractive, respectively. Similarly Assumptions \ref{assu-Martingale} and \ref{assu-Lipschitz} are also satisfied. Then, we have the following result.
\begin{theorem}\label{thm:expectation-special}
    Suppose Assumptions \ref{assu-Markov}-\ref{assu-Lipschitz} hold, and let $\alpha_n=\alpha_0/(n+1)$ and $\beta_n=\beta_0/(n+1)$. For iterates generated by \eqref{iter-special}, there exist constants $C_4,C_5, C_6$ and time $n_0$ such that if $\beta_0/\alpha_0\leq C_4<1$ and $\beta_0\geq C_5$, then for all $n\geq n_0$, $$\EE[\|x_n-\xstar\|^2+\|y_n-\ystar\|^2]\leq \frac{C_6}{n+1}.$$
    
\end{theorem}
The proof for this theorem differs only slightly from the proof for Theorem \ref{thm:expectation-bound}. The constants have been computed explicitly in the proof presented in Appendix II.

Note that we can drop the assumptions $\beta_0/\alpha_0\leq C_4$ if we choose $\alpha_n=\alpha_0/(n+1)^{1-\epsilon}$ for $\epsilon\in(0,0.5)$. In that case it can be shown that $\EE[\|x_n-\xstar\|^2+\|y_n-\ystar\|^2]= \mathcal{O}(1/(n+1)^{1-\epsilon})$. Here $\epsilon>0$ can be arbitrarily small, giving us a bound which can be made arbitrarily close to $\mathcal{O}(1/n)$. 

\section{Proof Outline for Theorem \ref{thm:expectation-bound}}\label{sec:outline}
In this section, we present a proof sketch for Theorem \ref{thm:expectation-bound} through a series of lemmas. The complete proof has been presented in Appendix I. The proof takes inspiration from the analysis of non-linear two-time-scale iterations which are contractive under Euclidean norm \cite{Doan}. But the presence of arbitrary norm contractions and Markov noise require additional tools which we discuss first.
\subsection{Moreau Envelopes}
The convergence rate analysis of SA by constructing smooth Lyapunov functions of the iterates. Consider a single-time-scale iteration $x_{n+1}=x_n+\alpha_n(h(x_n)-M_{n+1})$. For such an iteration, $\|x_n-\xstar\|^2$ acts as the Lyapunov function where $\|\cdot\|$ is the norm with respect to which $h(x)$ is contractive. Additionally, when this Lyapunov function is smooth with respect to some norm, we can obtain a finite-time bound for the algorithm. But for arbitrary norms (such as $\|\cdot\|_\infty$), $\|x_n-\xstar\|^2$ is not guaranteed to be smooth. In such cases, we need to define a smooth Lyapunov function. 

We first define the functions needed for our two-time-scale case and then verify that they satisfy the required properties.
\begin{definition}\label{defn:Moreau}
    For $x\in\RR^{d_1}$, $y\in\RR^{d_2}$ and $q>0$ define functions $\Acal:\RR^{d_1}\mapsto\RR$ and $\Bcal:\RR^{d_2}\mapsto\RR$ as 
\begin{align*}
\Acal(x)&=\inf_{v\in\RR^{d_1}}\left\{\frac{1}{2}\|v\|^2+\frac{1}{2q}\|x-v\|_2^2\right\}    \\
\Bcal(y)&=\inf_{w\in\RR^{d_2}}\left\{\frac{1}{2}\|w\|^2+\frac{1}{2q}\|y-w\|_2^2\right\}.   
\end{align*}
Here $\|\cdot\|_2$ denotes the $\ell_2$ norm. 
\end{definition}
Note that the dependence on $q$ has been rendered implicit in the definitions of $\Acal(\cdot)$ and $\Bcal(\cdot)$. We treat $q$ as a constant and we show in the proof for Lemma \ref{lemma:recursive} that choosing a small enough $q$ is sufficient for our proof. The choice of $q$ only affects the constants in Theorems \ref{thm:expectation-bound} and \ref{thm:expectation-special}.

Due to the norm equivalence in finite-dimensional real spaces, there exists $\ell_{d_1}\in(0,1]$ and $u_{d_1}\in[1,\infty)$ such that $\ell_{d_1}\|x\|\leq \|x\|_2\leq u_{d_1}\|x\|$ for all $x\in\RR^{d_1}$. We next reproduce a result from \cite{chenfinite}.
\begin{lemma}\label{lemma:chen-A}
    Function $\Acal(\cdot):\RR^{d_1}\mapsto\RR$ satisfies the following.
    \begin{enumerate}[label=(\alph*)]
        \item $\Acal$ is convex and $1/q$ smooth w.r.t. the $\ell_2$ norm, i.e.,
        $$\Acal(x_2)\leq \Acal(x_1)+\langle\nabla \Acal(x_1), x_2-x_1\rangle_2 + \frac{1}{2q}\|x_2-x_1\|^2_2, $$
        holds true for all $x_1,x_2\in\RR^{d_1}$.
        \item For all $x\in\RR^{d_1}$, $$\left(1+\frac{q}{u_{d_1}^2}\right)\Acal(x)\leq \frac{1}{2}\|x\|^2\leq \left(1+\frac{q}{\ell_{d_1}^2}\right)\Acal(x).$$
        \item There exists a norm, denoted by $\|\cdot\|_\Acal$, such that $\Acal(x)=\frac{1}{2}\|x\|_\Acal^2$ for all $x\in\RR^{d_1}$. 
    \end{enumerate}
\end{lemma}
For sufficiently small $q$, the first two parts of the lemma imply that $\Acal(\cdot)$ is a smooth approximation for the function $(1/2)\|\cdot\|^2$. The final part implies that $\Acal(\cdot)$ is in fact a scaled squared norm. Similar lemma for $\Bcal(\cdot)$ and other required results have been presented in Appendix I-A.

\subsection{Solutions of Poisson Equation}
Recall that the fixed point and the contractive properties are defined with respect to the functions $\barf(\cdot)$ and $\barg(\cdot)$. But the corresponding iterations contain $f(\cdot)$ and $g(\cdot)$, respectively. Hence the terms $f(x_n,y_n,Z_n)-\barf(x_n,y_n)$ and $g(x_n,y_n,Z_n)-\barg(x_n,y_n)$ act as the Markovian noise in our framework. To handle this error term, we define the following solutions of Poisson equation.
\begin{definition}\label{defn:VandW}
    For a fixed $i_0\in\Scal$, $\tau\coloneqq \min\{n>0: \!Z_n=i_0\}$ and $E_i[\cdot]=E[\cdot|Z_0=i]$, define $V(x,y,i)$ and $W(x,y,i)$ as
    \begin{align*}
        V(x,y,i)&=E_i\left[\sum_{m=0}^{\tau-1}\Big(f(x,y,Z_m)-\sum_{j\in\Scal}\pi(j)f(x,y,j)\Big)\right],\\
        W(x,y,i)&=E_i\left[\sum_{m=0}^{\tau-1}\Big(g(x,y,Z_m)-\sum_{j\in\Scal}\pi(j)g(x,y,j)\Big)\right],
    \end{align*}
for all $x\in\RR^{d_1},y\in\RR^{d_2}$ and $i\in\Scal$.
\end{definition}

For each $x,y$, $V(x,y,i)$ is a solution to the following Poisson equation \cite{Benaim}, i.e., $V(x,y,i)$ satisfies
\begin{align*}
    V(x,y,i)&=f(x,y,i)-\sum_{j\in\Scal}\pi(j)f(x,y,j)\\
    &\;\;\;\;\;\;\;\;\;\;\;\;\;\;\;\;\;\;\;\;\;\;\;\;\;\;\;\;\;+\sum_{j\in\Scal}p(j|i)V(x,y,j).
\end{align*}
Then the `error' $f(x_n,y_n,Z_n)-\barf(x_n,y_n)$ can be written as
\begin{align*}
   &f(x_n,y_n,Z_n)-\barf(x_n,y_n)\\
   &=V(x_n,y_n,Z_n)-\sum_{j\in\Scal}p(j|Z_n)V(x_n,y_n,j) \nonumber\\
   &=V(x_n,y_n,Z_{n+1})-\sum_{j\in\Scal}p(j|Z_n)V(x_n,y_n,j)\nonumber\\
   &\;\;\;+V(x_n,y_n,Z_n)-V(x_n,y_n,Z_{n+1}).
\end{align*}
For the first term, we define ${\tilde{V}_{n+1}}=V(x_n,y_n,Z_{n+1})-\sum_{j\in\Scal}p(j|Z_n)V(x_n,y_n,j)$. Note that $\{\tilde{V}_{n+1}\}$ is a martingale difference sequence with respect to filtration $\mathcal{F}_n$, and is combined with $M_{n+1}$. The second term is appropriately combined with other terms to create a telescopic series. Hence the above definition allows us to express the Markov noise in terms of a martingale difference sequence.

We present the analogous Poisson equation for $W(\cdot)$ along with the Lipschitz and linear growth properties for $V(\cdot)$ and $W(\cdot)$ in Appendix I-B.
\subsection{Proof Sketch}
We next present a series of lemmas which together complete the proof for Theorem \ref{thm:expectation-bound}. The first step in the proof is to show that $\xstar(y)$ is Lipschitz in $y$. 
\begin{lemma}\label{lem:xstar-Lipschitz}
    Suppose Assumptions \ref{assu-contractive-f} and \ref{assu-Lipschitz} hold. Then the map $y\mapsto\xstar(y)$ is Lipschitz with parameter $L_1\coloneqq L/(1-\lambda)$, i.e.,
    \begin{equation*}
      \|\xstar(y_1)-\xstar(y_2)\|\leq L_1\|y_1-y_2\|,
    \end{equation*}
    for $y_1,y_2\in\RR^{d_2}$.
\end{lemma}
The above lemma ensures that the target points for the faster time-scale iteration, $\xstar(y_n)$, move slowly as the iterates $y_n$ move slowly. The next step is to derive intermediate recursive bounds on $\EE[\Acal(x_{n}-\xstar(y_n))]$ and $\EE[\Bcal(y_n-\ystar)]$.
\begin{lemma}\label{lemma:recursive}
    Suppose the setting of Theorem \ref{thm:expectation-bound} holds. Then, for constants $\Gamma_1,\Gamma_2,\Gamma_3>0$, the following hold for all $n\geq 0$.
    \begin{align*}
        &\EE[\Acal(x_{n+1}-\xstar(y_{n+1}))]\leq (1-\alpha_n\lambda')\EE[\Acal(x_{n}-\xstar(y_n))]\\
        &+\EE\left[\alpha_n(d_n-d_{n+1})\right]\nonumber+\Gamma_1\left(\alpha_n^2+\frac{\beta_n^2}{\alpha_n}\right)\times\\
        &\left(1+\EE[\Acal(x_n-x^*(y_n))+\Bcal(y_n-y^*)]\right).\\
        &\EE[\Bcal(y_{n+1}-\ystar)]\leq(1-\beta_n\mu')\EE[\Bcal(y_{n}-\ystar)]\\
        &+\beta_n\EE[(e_n-e_{n+1})]+\beta_n\Gamma_3\EE[\Acal(x_n-\xstar(y_n))]\\
        &+\alpha_n\beta_n\Gamma_2(1+\EE[\Acal(x_n-x^*(y_n))+\Bcal(y_n-y^*)]).
    \end{align*}
    Here $d_n=\langle\Acal(x_n-\xstar(y_n)),V(x_n,y_n,Z_n)\rangle_2$, and $e_n=\langle\Bcal(y_n-\ystar),W(x_n,y_n,Z_n)\rangle_2$. 
\end{lemma}
Here the terms $\alpha_n(d_n-d_{n+1})$ and $\beta_n(e_n-e_{n+1})$ correspond to the telescoping series arising due to the Poisson equation solution approach to Markov noise. We add the two intermediate bounds above, and show that the iterates are bounded in expectation.
\begin{lemma}\label{lemma:bounded-expectation}
    Suppose the setting of Theorem \ref{thm:expectation-bound} holds. Then the iterates $x_k$ and $y_k$ are bounded in expectation. Specifically, there exists constant $\Gamma_4>0$ such that for all $k\geq 0$,
    $$\EE[\Acal(x_n-\xstar(y_n))+\Bcal(y_n-\ystar)]\leq \Gamma_4.$$
\end{lemma}

The above lemmas are sufficient to show that the iterates $x_n,y_n$ almost surely converge to the fixed point using the ODE approach to analysis of SA schemes. A requirement for the ODE approach is that the iterates should be almost surely finite. This is guaranteed by the above lemma as bounded expectation implies almost sure finiteness of iterates. We substitute the above bound on iterates into the intermediate bounds in Lemma \ref{lemma:recursive} to obtain the bounds on $\EE[\Acal(x_n-\xstar(y_n))]$ and $\EE[\Bcal(y_n-\ystar)]$.
\begin{lemma}\label{lemma:almost-done}
    Suppose the setting of Theorem \ref{thm:expectation-bound} holds. Then,
    \begin{enumerate}[label=(\alph*)]
        \item The iterates $(x_n,y_n)$ almost surely converge to the fixed points $(\xstar,\ystar)$.
        \item There exists $n_1\geq 0$ and constant $\Gamma_5>0$ such that for all $n\geq n_1$,
        $$
        \EE[\Acal(x_n-\xstar(y_n))]\leq \Gamma_5\left(\alpha_n+\frac{\beta_n^2}{\alpha_n^2}\right).$$
        \item Assume that $\beta_0\geq C_3$ as in Theorem \ref{thm:expectation-bound}. Then, there exists $n_2\geq 0$ and constant $\Gamma_6>0$ such that for all $n\geq n_2$,
        $$
        \EE[\Bcal(y_n-\ystar)]\leq \Gamma_6\left(\alpha_n+\frac{\beta_n^2}{\alpha_n^2}\right).$$
    \end{enumerate}
\end{lemma}

Using properties of Moreau envelopes and combining the above two bounds completes the proof for Theorem \ref{thm:expectation-bound}.

\section{Applications to Reinforcement Learning and Optimization}\label{sec:applications}
\subsection{SSP Q-Learning for Average Cost MDPs}
In this subsection, we apply our theorems to asynchronous SSP Q-Learning, a two time-scale stochastic approximation algorithm applied to the average cost problem \cite{abounadi, bert, tsitsiklis}. Consider a controlled Markov chain $\{X_n\}$ on a finite state space $\Scal'$, $|\Scal'|=\sfrak$, controlled by a control process $\{U_n\}$  in a finite action space $\UU$, $|\UU|=\ufrak$. The controlled transition probability function $(i,j,u) \in \Scal'^2\times \UU\mapsto p(j|i,u) \in [0,1]$ satisfies $\sum_jp(j | i,u) = 1 \ \forall \; i,u.$
Thus
$$\PP(X_{n+1} = j | X_m, U_m, m \leq n) = p(j | X_n, U_n), \ a.s.$$
We assume that the graph of the Markov chain remains irreducible under all control choices. Let $k: S\times U \mapsto \RR$ be a prescribed `running cost' function. The objective is to choose actions $U_n$ so as to minimize the average cost
\begin{align}\label{eq:avg_obj}
    \lim_{n\to\infty}\frac{1}{n}\sum_{m=0}^{n-1}\mathbb{E}[k(X_m,U_m)].
\end{align}

Let $i_0\in\Scal'$ be some fixed state, referred to as the reference state. Then the SSP Q-learning algorithm is:
\begin{subequations}\label{SSP}
    \begin{gather}
        Q_{n+1}(i,u)=Q_n(i,u)+\alpha_n I\{X_n=i,U_n=u\} 
        \Big(k(i,u)\nonumber\\+I\{X_{n+1}\neq i_0\}\min_{v}Q_n(X_{n+1},v)-Q_n(i,u)-\rho_n\Big),\label{SSP-Q}\\
        \rho_{n+1}=\rho_n+\beta'_n(\min_v Q_n(i_0,v)).\label{SSP-lambda}
    \end{gather}
\end{subequations}
Here $Q_n\in\RR^{\sfrak\ufrak}, \rho_n\in\RR$ with any arbitrary $Q_0(\cdot,\cdot)\geq0$ and $\rho_0$. $I\{\cdot\}$ denotes the indicator function, and $I\{X_n=i,Z_n=u\}$ stems from the asynchronous update, i.e., at any time $n$, Q-values are updated only for the current state, action pair $\{X_n,U_n\}$. Note that unlike \cite{bert,abounadi} we do not use the projection operator on the iteration for $\rho_n$. 

We study the offline SSP Q-Learning algorithm, where the $U_n$ are decided by a randomized sampling policy $\Phi_s(u\mid i)$ independent of $Q_n$, i.e.,  
\begin{align*}
    &\PP(U_n=u\mid X_m, Q_m,U_k,m\leq n,k<n)\\
    &\;\;\;\;\;\;\;\;\;\;\;\;\;\;\;\;\;\;\;\;\;\;\;\;\;\;\;\;\;\;=\Phi_s(u|X_n), u\in\UU,n\geq 0.
\end{align*}
We further assume that $\Phi_s(u|i)>0$ for all $u\in\UU$ and $i\in\Scal'$. For the purpose of applying our Theorem \ref{thm:expectation-bound}, we define $Z_n=(X_n,U_n)$, i.e., the Markov chain in our formulation will be given by the controlled Markov chain $X_n$ and the control $U_n$ together and the state space is $\Scal=\Scal'\times \UU$. The transition probabilities are then given by 
\begin{align*}
    \PP(Z_{n+1}=(i,u)\mid Z_n) = p(i|X_n,U_n)\Phi_s(u|i).
\end{align*}
The stationary distribution for $\{Z_n\}$ is given by $\pi(i,u) =\pi_{\Phi}(i)\Phi(u|i)$. Here $\pi_{\Phi}(\cdot)$ denotes the stationary distribution of the controlled Markov chain $\{X_n\}$ under control $\Phi_s(\cdot|\cdot)$. 

We first rewrite the iterations \eqref{SSP} in the form of \eqref{iter-special}. 
 \begin{align*}
   Q_{n+1}(i,u)&=Q_n(i,u)+\alpha_n(f^{i,u}(Q_n,\rho_n,Z_n)\nonumber\\&~~-Q_{n}(i,u)+M_{n+1}(i,u))\\
    \rho_{n+1}&=\rho_n+\beta_n(\barg(Q_n, \rho_n)-\rho_n).
 \end{align*}
Here, 
\begin{align}\label{SSP-f-g-defn}
    &f^{i,u}(Q,\rho,Z)=I\{Z=(i,u)\}\Big(k(i,u)-Q(i,u)-\rho\nonumber\\&~~+\sum_{j\neq i_0}p(j|i,u)\min_vQ(j,v)\Big)+Q(i,u),\nonumber\\
    &\barg(Q,\rho)=\beta'\min_vQ(i_0,v)+\rho,\nonumber\\
    &M_{n+1}^{i,u}=I\{Z_n=(i,u)\}\Big(I\{X_{n+1}\neq i_0\}\times\nonumber\\
    &~~\min_{v}Q_n(X_{n+1},v)-\sum_{j\neq i_0}p(j|i,u)\min_vQ(j,v)\Big).
\end{align}
Note that $\beta_n=\beta'_n/\beta'$. Here $\beta'>0$ is a sufficiently small constant, used in Proposition \ref{prop:SSP}. Recall that $\barf(Q,\rho)=\sum_{(i,u)}\pi(i,u)f(Q,\rho,(i,u))$. Note that this algorithm is noiseless in the slower time-scale and hence, the formulation from Section \ref{sec:special} can be applied here.

Our following proposition details the contractivity and fixed point properties of functions $\barf(\cdot)$ and $\barg(\cdot)$. For this, define the norm $\|\cdot\|_w$ as the weighted max-norm. For $x\in\RR^d$, $\|x\|_w=\max_j|w_jx_j|$ for positive $w_1,\ldots,w_d$.

\begin{proposition}\label{prop:SSP}
Functions $f(\cdot)$ and $\barg(\cdot)$ defined in \eqref{SSP-f-g-defn} satisfy the following.
\begin{enumerate}[label=(\alph*)]
    \item The function $\barf(Q,\rho)$ is contractive in $Q$ under norm $\|\cdot\|_w$ for some $w$, i.e., there exists some $\lambda\in(0,1)$ and $w\in\RR^{\sfrak\ufrak}$ such that $\|\barf(Q_1,\rho)-\barf(Q_2,\rho)\|_w\leq \lambda\|Q_1-Q_2\|_w$ for all $\rho\in\RR$. 
    \item Let $\Qstar(\rho)$ denote the fixed point for $\barf(\cdot,\rho)$. Then, for sufficiently small $\beta'$, there exists $\mu\in(0,1)$ such that $|\barg(\Qstar(\rho_2),\rho_2)-\barg(\Qstar(\rho_1),\rho_1)|\leq \mu|\rho_2-\rho_1|$.
    \item The fixed point $\rhostar$ for $\barg(\Qstar(\cdot),\cdot)$ is the minimum of the average cost defined in \eqref{eq:avg_obj} and $\Qstar=\Qstar(\rhostar)$ is the vector of Q-values associated with the optimal policy.
\end{enumerate}
\end{proposition}

We finally present the main result for SSP Q-Learning. 
\begin{corollary}\label{coro:SSP}
    Let $\alpha_n=\alpha_0/(n+1)$ and $\beta_n=\beta_0/(n+1)$. Then there exist constants $C_4', C_5', n_0'$ and $w\in\RR^{\sfrak\ufrak}$ such that if $\beta_0/\alpha_0\leq C_4'$ and $\beta_0\geq C_5'$ then for all $n\geq n_0'$
    $$\EE[\|Q_n-\Qstar\|^2_w]= \mathcal{O}\left(\frac{1}{n+1}\right)$$ and $$\EE[|\rho_n-\rho^*|^2]=\mathcal{O}\left(\frac{1}{n+1}\right).$$
\end{corollary}
The proofs for the above proposition and corollary have been presented in Appendix III-A. By application of Theorem \ref{thm:expectation-bound}, we can also show almost sure convergence of iterates $(Q_n,\rho_n)$ to fixed points $(\Qstar,\rhostar)$ when $\alpha_n=\alpha_0/(n+1)^{1-\epsilon}$ where $\epsilon\in(0,0.5)$.

\subsection{Discounted-Reward Q-Learning with Polyak-Averaging}\label{subsec:Q-learning}
Define the controlled Markov chain $\{X_n\}$, the control $\{U_n\}$, the transition function $p(\cdot|\cdot,\cdot)$ and the cost $k(\cdot,\cdot)$ as in SSP Q-Learning. In the discounted reward setting, the objective is to minimize the discounted cost 
\begin{align*}
    \sum_{m=0}^{\infty}\mathbb{E}[\gamma^mk(X_m,U_m)],
\end{align*}
where $\gamma\in(0,1)$ is the discount factor. One of the popularly used algorithms to find the optimal policy is Q-Learning, proposed by \cite{watkins1992}. In order to improve the statistical efficiency and the rate of convergence, an additional step of iterate averaging is used, called Polyak-Ruppert averaging \cite{ruppert1988efficient, polyak1992acceleration}. The combined iteration is given as
\begin{subequations}\label{eq:Q-learning}
    \begin{align}
        Q_{n+1}(i,u)&=Q_n(i,u)+\alpha_nI\{X_n=i, U_n=u\}\Big(k(i,u)\nonumber\\
        &~~+\gamma\min_{v}Q_n(X_{k+1}, v)-Q_n(i,u)\Big)\\
        \bar{Q}_{n+1}(i,u)&=\bar{Q}_n(i,u)+\frac{\beta_0}{n+1}\left(Q_n(i,u)-\bar{Q}_n(i,u)\right)
    \end{align}
\end{subequations}
where $\{(X_n, U_n)\}_{n\geq 0}$ is a sample trajectory collected using a sampling policy $\Phi_s(u|i)$. Note that the above iterations can be seen as a special case of two-time-scale SA where the slower time-scale iterate, $\bar{Q}_n$, does not affect the faster time-scale iterate, $Q_n$. Assuming $\Phi_s(u|i)>0$ for all $u\in \UU$ and $i\in \Scal'$, we again construct the Markov chain $Z_n=(X_n, U_n)$ whose transition probability is given by $\PP(Z_{n+1}=(i,u)\mid Z_n)=p(i|X_n,Z_n)\Phi_s(u|i)$. Let $\pi(i,u)$ denote the stationary distribution for the Markov chain $\{Z_n\}$ as in SSP Q-Learning. Rearranging the iterations in \eqref{eq:Q-learning}, we get
\begin{subequations}\label{eq:re Q-learning}
    \begin{align}
        Q_{n+1}(i,u)&=Q_n(i,u)+\alpha_n(f^{i,u}(Q_n, Z_n)\nonumber\\
        &~~-Q_n(i,u)+M_{n+1}^{i,u}\\
        \bar{Q}_{n+1}&=\bar{Q}_n+\beta_n(\barg(Q_n,\bar{Q}_n)-\bar{Q}_n(i,u)).
    \end{align}
\end{subequations}
Here,
\begin{align*}
    &f^{i,u}(Q, \bar{Q}, Z)=I\{Z=(i,u)\}\Big(k(i,u)-Q(i,u)\\
    &~~+\gamma \sum_{j\in \Scal'}p(j|i,u)\min_vQ(j,v)\Big)+Q(i,u),\\
    &\barg(Q, \bar{Q})=Q,\\
    &M_{n+1}^{i,u}=\gamma I\{Z_n=(i,u)\}\Big(Q_n(X_{n+1},v)\\
    &~~-\sum_{j\in\Scal'}p(j|i,u)\min_vQ(j,v)\Big),
\end{align*}
and $\beta_n=\beta_0/(n+1)$. As the slower time-scale is noiseless, we apply the formulation from Section \ref{sec:special} here. Since $\bar{Q}$ does not affect the faster time-scale, the fixed point for $\barf(\cdot,\bar{Q})$, $\Qstar(\bar{Q})$, is independent of $\bar{Q}$. Let $\Qstar$ denote this fixed point, which is the vector of Q-values corresponding to the optimal policy, and satisfies the following for all $i\in\Scal',u\in\UU$.
\begin{equation}\label{Qstar-polyak}
    \Qstar(i,u)=k(i,u)+\gamma\sum_{j\in\Scal'}p(j|i,u)\min_{v\in\UU}\Qstar(j,v).
\end{equation}
Note that $\barg(Q,\bar{Q})=Q$, and hence $\barg(\Qstar(\bar{Q}),\bar{Q})=\Qstar(\bar{Q})=\Qstar$. Using Section 5.1 from \cite{Chandak_conc}, we have that $\barf(Q,\bar{Q})$ is contractive in $Q$ under the norm $\|\cdot\|_\infty$ with contraction factor $\lambda=1-\pi_{min}(1-\gamma)$ where $\pi_{min}=\min_{i,u}\pi(i,u)$. On the other hand, the contraction factor for $\barg(\Qstar(\cdot),\cdot)$ is simply zero.

\begin{corollary}\label{coro:Q-polyak}
    Let $\alpha_n=\alpha_0/(n+1)$ and $\beta_n=\beta_0/(n+1)$. Then there exist constants $C_4''$ and $n_0''$ such that if $\beta_0/\alpha_0\leq C_4''$ and $\beta_0\geq 8$ then, for all $n\geq n_0''$ 
    $$\EE[\|\bar{Q}_n-\Qstar\|_{\infty}^2]=\mathcal{O}\left(\frac{1}{n+1}\right).$$
\end{corollary}
The proof for the above corollary has been presented in Appendix III-B. Similar to SSP Q-Learning, almost sure convergence of $Q_n$ and $\bar{Q}_n$ to $\Qstar$ can be shown when $\alpha_n=\alpha_0/(n+1)^{1-\epsilon}$ where $\epsilon\in(0,0.5)$.

\begin{remark}
In the original Polyak-Ruppert averaging, $\beta_0$ corresponds to $1$ which is equivalent to the time averaging of iterates. One of the main objectives of this averaging step is to achieve $\mathcal{O}(1/n)$ rate of convergence without requiring knowledge of system parameters to tune the step-size (see Corollary 2.1.2 in \cite{chen2024lyapunov}). While the only requirement on $\beta_0$ is that $\beta_0\geq 8$ which is independent of the system parameters, $\alpha_0$ needs to be chosen according to $\alpha_0\geq \beta_0/C_4''$, which requires knowledge of the properties of the Markov chain.
    At the expense of a larger $n_0''$, one can modify our proof to allow for $\beta_0$ arbitrarily close to 1. An interesting future direction is showing the optimal rate of convergence for Polyak-Ruppert averaging of nonlinear stochastic approximation with Markovian noise.     
\end{remark}

\subsection{Learning GNE in Strongly Monotone Games}
Consider a set of $K$ players $\{1, \ldots, K\}$. Each player $k$ takes action $x^{(k)}_n\in\RR^{d}$ at each time $n$. We use $\xx_n=(x^{(1)}_n,\ldots,x^{(K)}_n)$ to denote the $Kd$ dimensional concatenation of all players' actions. Each player $k$ receives utility $u_k(\xx_n)$ which is a function of all players' actions. The players want to converge to the Nash equilibrium under $c$ linear constraints $A\xx=b$ where $A\in\RR^{c\times Kd}$ and $b\in\RR^c$. Note that these constraints are coupled across players' actions. This is known as the Generalized Nash Equilibrium Problem (GNEP) \cite{GNEP}. Efficient computation of GNEs has been studied in the noiseless setting (e.g., \cite{GNE-noiseless-1}, \cite{GNE-noiseless-2}). Learning of GNEs for stochastic games has been studied in settings where the players share information about their actions or rewards with each other \cite{GNE-central-1, GNE-central-2}. The problem of learning GNE has also been formulated in the setting of game control \cite{Chandak-game}.

We study learning of GNE in the stochastic and distributed setting. At each timestep, player $k$ observes a noisy gradient of their utility, i.e., $\nabla_{x^{(k)}}u_k(\xx)+M^{(k)}_{n+1}$ and the linear constraint violation $A\xx-b+M'_{n+1}$, which is common for all players. Note that there is no Markovian noise in this setting. Before making assumptions on the structure of the game, we define the gradient operator $$F(\xx)\coloneqq (\nabla_{x^{(1)}}u_1(\xx),\ldots,\nabla_{x^{(K)}}u_K(\xx)).$$ We make the following assumption on the game:
\begin{assumption}\label{assu-game}
    The game is strongly monotone, i.e., there exists $\lambda_0$ such that for all $\xx_2,\xx_1\in\RR^{Kd}$ $$\langle \xx_2-\xx_1,F(\xx_2)-F(\xx_1)\rangle\leq -\lambda_0\|\xx_2-\xx_1\|_2^2.$$
    Additionally, we assume that $F(\cdot)$ is $\ell-$Lipschitz. 
\end{assumption}

Being a strongly monotone game, the players can converge to the Nash equilibrium by just gradient play, i.e., by performing gradient ascent on their utilities. To satisfy the additional coupled linear constraints, each player $k$ performs the following two-time-scale iteration.
\begin{equation}\label{GNE-learning-playerwise}
    \begin{aligned}
x^{(k)}_{n+1}&=x^{(k)}_n+\alpha'_n (\nabla_{x^{(k)}}u_k(\xx_n)+M^{(k)}_{n+1}-B^{(k)}y_n)\\
y_{n+1}&=y_n+\beta'_n (A\xx_n-b+M'_{n+1}).
\end{aligned}
\end{equation}
Here $B=A^T$ and $B^{(n)}$ denotes the row relevant to player $k$ in the matrix $A^T$ and $y_n\in\RR^c$. Note that the above two-time-scale iteration can be thought of as Lagrangian optimization under the constraint $A\xx=b$ where $y_n$ acts as the Lagrange multiplier. Combining iteration \eqref{GNE-learning-playerwise} for all players, we can write the above algorithm as the following set of iterations. 
\begin{equation}\label{GNE-learning-combined}
    \begin{aligned}
\xx_{n+1}&=\xx_n+\alpha'_n (F(\xx_n)-By_n+M_{n+1})\\
y_{n+1}&=y_n+\beta'_n (A\xx_n-b+M'_{n+1}).
\end{aligned}
\end{equation}
Note that the algorithm is still distributed and \eqref{GNE-learning-combined} is just a concise method of representing the iterations. Similar to the other applications, we rearrange these coupled iterations.
\begin{equation}\label{GNE-learning-iter}
    \begin{aligned}
\xx_{n+1}&=\xx_n+\alpha_n (f(\xx_n,y_n)-\xx_n+M''_{n+1})\\
y_{n+1}&=y_n+\beta_n (g(\xx_n,y_n)-y_n+M'''_{n+1}).
\end{aligned}
\end{equation}
Here, $\alpha_n=\alpha'_n/\alpha'$ and $\beta_n=\beta'_n/\beta'$ are scaled stepsize sequences and $M''_{n+1}=\alpha'M_{n+1}$ and $M'''_{n+1}=\beta'M'_{n+1}$ are scaled martingale noise sequences. Note that $f(\xx,y)=\xx+\alpha'F(\xx)-\alpha'By$ and $g(\xx,y)=y+\beta'A\xx-\beta'b$.

Let $\xx^*(y)$ denote the unique solution to $F(\xx)=By$ (i.e., $F(\xx^*(y))=By$). This is the Nash equilibrium corresponding to parameter $y$. The existence and uniqueness of this Nash equilibrium is guaranteed by Proposition \ref{prop:GNEP} below. Our following proposition states the required contraction and fixed point properties of above mappings.
\begin{proposition}\label{prop:GNEP}
    Suppose that matrix $A$ has full row rank and that Assumption \ref{assu-game} holds. Then,
    \begin{enumerate}[label=\alph*)]
        \item For sufficiently small $\alpha'$, there exists $0<\lambda<1$ such that $f(\xx,y)$ is $\lambda$-contractive in $\xx$ under the $\ell_2$ norm.
        \item For sufficiently small $\beta'$, there exists $0<\mu<1$ such that $g(\xx^*(y),y)$ is $\mu$-contractive under the $\ell_2$ norm.
        \item There exists unique $(\xx^*,\ystar)$ such that $A\xx^*=b$ and $F(\xx^*)=B\ystar$.
    \end{enumerate}
\end{proposition}
Explicit values for all constants above have been computed in the proof presented in Appendix III-C. We now present the main result for this algorithm.

\begin{corollary}\label{coro:GNEP}
Suppose that Assumption \ref{assu-game} holds and that noise sequences $\{M''_{n+1}\}$ and $\{M'''_{n+1}\}$ satisfy Assumption \ref{assu-Martingale}. Moreover, let the matrix $A$ have full row rank. For $\alpha_n=\alpha_0/(n+1)^{2/3}$ and $\beta_n=\beta_{0}/(n+1)$, iterates $(x_n,y_n)$ almost surely converge to the fixed points $(\xx^*,\ystar)$. Furthermore, there exists constant $C_3'$ such that for $\beta_0\geq C_3'$, 
    \begin{align*}
        \EE\left[\|A\xx_n-b\|_2^2+\|\xx_n-\xx^*\|_2^2\right]&=\mathcal{O}\left(1/(n+1)^{2/3}\right).
    \end{align*}
\end{corollary}
The proof for the above corollary has been presented in Appendix III-C.

\section{Conclusion}\label{sec:conc}
In this paper, we obtained finite-time bounds for two-time-scale SA where each time scale is a fixed point iteration, and the fixed point mapping is contractive with respect to arbitrary norms. We additionally incorporate Markovian noise into our framework to allow for analysis of RL algorithms. We show that the mean square error decays at a rate of $\mathcal{O}(1/n^{2/3})$. Our analysis uses Moreau envelopes and solutions of Poisson equations to deal with arbitrary norm contractions and Markovian noise, respectively. We further note that two-time-scale algorithms such as SSP Q-learning for average cost MDPs and Q-learning with Polyak averaging lack noise in the slower time-scale allowing for better convergence rates of $\mathcal{O}(1/k)$. 

Future directions include obtaining $\mathcal{O}(1/k)$ bounds in the general case where both time-scales have noise. A possible approach is to extend the noise-averaging technique of \cite{Chandak-1/k} to arbitrary norm contractions and Markov noise. An open problem in non-linear two-time-scale SA is whether a finite-time mean-square bound of $\mathcal{O}(1/k)$ can be established when the time-scales are truly separated, i.e., when $\lim_{k\uparrow\infty}\beta_k/\alpha_k=0$. Another promising direction is to establish concentration bounds for two-time-scale iterations, which could yield further insight into their behavior.

\newpage
\onecolumn
\appendices

\section{Proof for Theorem \ref{thm:expectation-bound}}\label{app:proof}
\subsection{\textbf{Properties of Moreau Envelopes}}\label{app:proof-Moreau}
We first present the analog of Lemma \ref{lemma:chen-A} for the function $\Bcal(\cdot)$. For this, note that there exists constants $\ell_{d_2}\in(0,1]$ and $u_{d_2}\in[1,\infty)$ such that $\ell_{d_2}\|y\|\leq \|y\|_2\leq u_{d_2}\|y\|$ for all $y\in\RR^{d_2}$.
\begin{lemma}\label{lemma:chen-B}
    Function $\Bcal(\cdot):\RR^{d_2}\mapsto\RR$ satisfies the following.
    \begin{enumerate}[label=(\alph*)]
        \item $\Bcal$ is convex and $1/q$ smooth w.r.t. the $\ell_2$ norm, i.e.,
        $$\Bcal(y_2)\leq \Bcal(y_1)+\langle\nabla \Bcal(y_1), y_2-y_1\rangle_2 + \frac{1}{2q}\|y_2-y_1\|^2_2, \;\forall y_1,y_2\in\RR^{d_2}.$$
        \item For all $y\in\RR^{d_2}$, $$\left(1+\frac{q}{u_{d_2}^2}\right)\Bcal(y)\leq \frac{1}{2}\|y\|^2\leq \left(1+\frac{q}{\ell_{d_2}^2}\right)\Bcal(y).$$
        \item There exists a norm, denoted by $\|\cdot\|_\Bcal$, such that $\Bcal(y)=\frac{1}{2}\|y\|_\Bcal^2$ for all $y\in\RR^{d_2}$. 
    \end{enumerate}
\end{lemma}
Lemmas \ref{lemma:chen-A} and \ref{lemma:chen-B} both follow from Lemma 2.1 from \cite{chenfinite}, where we use $L=1$ because we are working with the $\ell_2$ norm in the definition of $\Acal(\cdot)$ and $\Bcal(\cdot)$. We also quote additional important properties of the function $\Acal(\cdot)$ and $\Bcal(\cdot)$ from \cite{chenfinite}. 
\begin{lemma}\label{lemma:subg}
    For $\Acal(\cdot)$ and $\Bcal(\cdot)$ as defined in Definition \ref{defn:Moreau}, the following hold true.
    \begin{enumerate}[label=(\alph*)]
        \item  For $x_1,x_2\in\RR^{d_1}$,  $\langle \nabla \Acal(x_1),x_2\rangle_2 \leq \|x_1\|_\Acal\|x_2\|_\Acal.$ Similarly, for $y_1,y_2\in\RR^{d_2}$,  $\langle \nabla \Bcal(y_1),y_2\rangle_2 \leq \|y_1\|_\Bcal\|y_2\|_\Bcal.$
        \item For $x\in\RR^{d_1}$, 
        $\langle \nabla \Acal(x), x\rangle_2\geq 2\Acal(x).$ Similarly, for $y\in\RR^{d_2}$, 
        $\langle \nabla \Bcal(y), y\rangle_2\geq 2\Bcal(y).$
    \end{enumerate}
\end{lemma}

For ease of notation, we define $\ell=\min\{\ell_{d_1}, \ell_{d_2}\}$ and $u=\max\{u_{d_1}, u_{d_2}\}$. Then for $x\in\RR^{d_1}$ and $y\in\RR^{d_2}$, 
$$\ell \|x\|\leq \|x\|_2\leq u \|x\|\quad\text{and}\quad\ell\|y\|\leq \|y\|_2\leq u\|y\|,$$ and $$\left(1+q/u^2\right)\Acal(x)\leq \frac{1}{2}\|x\|^2\leq \left(1+q/\ell^2\right)\Acal(x)\quad\text{and}\quad \left(1+q/u^2\right)\Bcal(y)\leq \frac{1}{2}\|y\|^2\leq \left(1+q/\ell^2\right)\Bcal(y).$$ Equivalently, 
$$\sqrt{1+q/u^2}\|x\|_\Acal\leq \|x\|\leq \sqrt{1+q/\ell^2}\|x\|_\Acal\quad\text{and}\quad\sqrt{1+q/u^2}\|y\|_\Acal\leq \|y\|\leq \sqrt{1+q/\ell^2}\|y\|_\Bcal.$$

\subsection{\textbf{Properties of Solutions of Poisson Equation}}\label{app:proof-VandW}
Similar to $V(\cdot)$, for each $x$ and $y$, note that $W(x,y,i)$ is a solution to the following Poisson equation \cite{Benaim}, i.e., $W(x,y,i)$ satisfies
$$W(x,y,i)=g(x,y,i)-\sum_{j\in\Scal}\pi(j)g(x,y,j)+\sum_{j\in\Scal}p(j|i)W(x,y,j).$$

The following lemma shows that $V(x,y,i)$ and $W(x,y,i)$ are Lipschitz in $x,y$ and show linear growth.
\begin{lemma}\label{lem:poisson_soln-prop}
    The solutions of Poisson equation $V(\cdot,\cdot,\cdot)$ and $W(\cdot,\cdot,\cdot)$ as defined in Definition \ref{defn:VandW} satisfy linear growth and Lipschitzness, i.e., the following hold.
    \begin{enumerate}[label=(\alph*)]
        \item $\|V(x,y,i)\|+\|W(x,y,i)\|\leq K'(1+\|x\|+\|y\|)\; \forall x\in\RR^{d_1}, y\in\RR^{d_2}, i\in\Scal.$
        \item $\|V(x_1,y_1,i)-V(x_2,y_2,i)\|\leq L_2(\|x_1-x_2\|+\|y_1-y_2\|) \; \forall x_1,x_2\in\RR^{d_1}, y_1,y_2\in\RR^{d_2}, i\in\Scal.$
        \item $\|W(x_1,y_1,i)-W(x_2,y_2,i)\|\leq L_2(\|x_1-x_2\|+\|y_1-y_2\|) \; \forall x_1,x_2\in\RR^{d_1}, y_1,y_2\in\RR^{d_2}, i\in\Scal.$
    \end{enumerate}
\end{lemma}
\begin{proof}
    \begin{enumerate}[label=(\alph*)]
        \item Note that 
        \begin{align*}
    \|V(x,y,i)\|&= \left\|E_i\left[\sum_{m=0}^{\tau-1}(f(x,y,Z_m)-\sum_{j\in\Scal}\pi(j)f(x,y,j))\right]\right\|\\
    &\stackrel{(a)}{\leq}  E_i\left[\left\|\sum_{m=0}^{\tau-1}(f(x,y,Z_m)-\sum_{j\in\Scal}\pi(j)f(x,y,j))\right\|\right]\\
    &\stackrel{(b)}{\leq}  E_i\left[\sum_{m=0}^{\tau-1}\|f(x,y,Z_m)\|+\sum_{m=0}^{\tau-1}\left\|\sum_{j\in\Scal}\pi(j)f(x,y,j)\right\|\right]\\
    &\stackrel{(c)}{\leq} E_i[2\tau K(1+\|x\|+\|y\|)]= 2K(1+\|x\|+\|y\|)E_i[\tau].
\end{align*}
Here inequality (a) follows from Jensen's inequality, (b) follows from triangle inequality and (c) follows from Assumption \ref{assu-Lipschitz}. For an irreducible Markov chain with a finite state space, $E_i[\tau]$ is finite for all $i$. Hence, for all $i\in\Scal$, $\|V(x,y,i)\|\leq 2K\max_iE_i[\tau](1+\|x\|+\|y\|)$. Similarly $\|W(x,y,i)\|\leq 2K\max_iE_i[\tau](1+\|x\|+\|y\|)$, completing the proof for part (a), with $K'\coloneqq4K\max_iE_i[\tau]$.
\item Next, we show that $V$ is Lipschitz in $x$ and $y$.
\begin{align*}
    \|V(x_1,y,i)-V(x_2,y,i)\|&\leq E_i\Bigg[\sum_{m=0}^{\tau-1}\|f(x_1,y,Z_m)-f(x_2,y,Z_m)\|\\
    &\;\;\;\;\;\;\;\;+\sum_{m=0}^{\tau-1}\Bigg\|\sum_{j\in\Scal}\pi(j)(f(x_1,y,j)-f(x_2,y,j))\Bigg\|\Bigg]\\
    &\stackrel{(a)}{\leq} 2L\|x_1-x_2\|E_i[\tau].
\end{align*}
Here inequality (a) follows from Assumption \ref{assu-Lipschitz}. Then, for all $x_1,x_2,y,i$, $\|V(x_1,y,i)-V(x_2,y,i)\|\leq L_2\|x_1-x_2\|$, where $L_2=2\max_iE_i[\tau]L$. Similarly, for all $x,y_1,y_2,i$, $\|V(x,y_1,i)-V(x,y_2,i)\|\leq L_2\|y_1-y_2\|$. This completes the proof for part (b). 
\item The proof for part (c) follows exactly as part (b).
\end{enumerate}
\end{proof}

\subsection{\textbf{Proof for Lemma \ref{lem:xstar-Lipschitz} --- Lipschitz Nature of $\xstar(y)$}}
\begin{proof}[Proof for Lemma \ref{lem:xstar-Lipschitz}]
    Note that $\barf(\xstar(y),y)=\xstar(y)$. Then, for $y_1,y_2$
\begin{align*}
    \|\xstar(y_1)-\xstar(y_2)\|&=\|\barf(\xstar(y_1),y_1)-\barf(\xstar(y_2),y_2)\|\\
    &\leq \|\barf(\xstar(y_1),y_1)-\barf(\xstar(y_2),y_1)\|+ \|\barf(\xstar(y_2),y_1)-\barf(\xstar(y_2),y_2)\|\\
    &\leq \lambda\|\xstar(y_1)-\xstar(y_2)\|+L\|y_1-y_2\|
\end{align*}
This implies that $\|\xstar(y_1)-\xstar(y_2)\|\leq (L/(1-\lambda))\|y_1-y_2\|. $
\end{proof}

\subsection{\textbf{Proof for Lemma \ref{lemma:recursive}(a) --- Intermediate Recursive Bound on $\EE[\Acal(x_n-\xstar(y_n))]$}}
Let $n>0$ and $\lambda'=1-\lambda\sqrt{\frac{(1+q/\ell^2)}{(1+q/u^2)}}$. The constant $q$ is chosen such that $\lambda'>0$, which is always possible as $\sqrt{\frac{(1+q/\ell^2)}{(1+q/u^2)}}$ goes to one as $q$ goes to zero.
By substituting $x_2=x_{n+1}-\xstar(y_{n+1})$ and $x_1=x_n-\xstar(y_n)$ in Lemma \ref{lemma:chen-A}(a), we get   
\begin{subequations}\label{x-split}
\begin{align}
    \Acal(x_{n+1}-\xstar(y_{n+1}))&\leq \Acal(x_{n}-\xstar(y_n))\nonumber\\
    &\;\;+\langle\nabla \Acal(x_n-\xstar(y_n)),x_{n+1}-x_n+\xstar(y_n)-\xstar(y_{n+1})\rangle_2\label{x-split1}\\
    &\;\;+\frac{1}{2q}\|x_{n+1}-x_n+\xstar(y_n)-\xstar(y_{n+1})\|_2^2.\label{x-split2}
\end{align}
\end{subequations}
We further simplify the term in \eqref{x-split1} as follows.
\begin{subequations}\label{x1-split}
\begin{align}
    &\langle\nabla \Acal(x_n-\xstar(y_n)),x_{n+1}-x_n+\xstar(y_n)-\xstar(y_{n+1})\rangle_2\nonumber\\
    &=\alpha_n\langle\nabla \Acal(x_n-\xstar(y_n)),f(x_n,y_n,Z_n)-x_n+M_{n+1}\rangle_2+\left\langle\nabla \Acal(x_n-\xstar(y_n)),\xstar(y_n)-\xstar(y_{n+1})\right\rangle_2\nonumber\\
    &=\alpha_n\langle\nabla \Acal(x_n-\xstar(y_n)),\barf(x_n,y_n)-\barf(\xstar(y_n),y_n)\rangle_2\label{x1-split1}\\
    &\;\;-\alpha_n\langle\nabla \Acal(x_n-\xstar(y_n)),x_n-\xstar(y_n)\rangle_2\label{x1-split2}\\
    &\;\;+\alpha_n\langle\nabla \Acal(x_n-\xstar(y_n)),f(x_n,y_n,Z_n)-\barf(x_n,y_n)+M_{n+1}\rangle_2\label{x1-split3}\\
    &\;\;+\left\langle\nabla \Acal(x_n-\xstar(y_n)),\xstar(y_n)-\xstar(y_{n+1})\right\rangle_2\label{x1-split4}
\end{align}
\end{subequations}
The second equality follows from the definition that $\xstar(y_n)$ is a fixed point for $\barf(\cdot,y_n)$ and hence $\xstar(y_n)=\barf(\xstar(y_n),y_n)$. The last two terms (\eqref{x1-split3} and \eqref{x1-split4}) are `error' terms and we later show that they are negligible as compared to the other terms. We now bound the four terms in \eqref{x1-split} and the term \eqref{x-split2} as follows.
\begin{itemize}
    \item \textbf{Term (\ref{x1-split1}) ---} Using Lemma \ref{lemma:subg}(a), note that
$$\langle\nabla \Acal(x_n-\xstar(y_n)),\barf(x_n,y_n)-\barf(\xstar(y_n),y_n)\rangle_2\leq \|x_n-\xstar(y_n)\|_\Acal\|\barf(x_n,y_n)-\barf(\xstar(y_n),y_n)\|_\Acal.$$
Now, using Lemma \ref{lemma:chen-A}, we have
\begin{align*}
    \Acal(\barf(x_n,y_n)-\barf(\xstar(y_n),y_n))&\leq \frac{1}{2(1+q/u^2)}\|\barf(x_n,y_n)-\barf(\xstar(y_n),y_n)\|^2\\
    &\stackrel{(a)}{\leq} \frac{\lambda^2}{2(1+q/u^2)} \|x_n-\xstar(y_n)\|^2\leq \frac{\lambda^2(1+q/\ell^2)}{(1+q/u^2)} \Acal(x_n-\xstar(y_n)).
\end{align*}
Here inequality (a) follows from the assumption that $\barf(\cdot,y)$ is contractive (Assumption \ref{assu-contractive-f}). 
Using the fact that $\Acal(x)=1/2\|x\|_\Acal^2$, we now have
\begin{align}\label{x1-split1-soln}
    &\langle\nabla \Acal(x_n-\xstar(y_n)),\barf(x_n,y_n)-\barf(\xstar(y_n),y_n)\rangle_2\nonumber\\
    &\leq\|x_n-\xstar(y_n)\|_\Acal\|\barf(x_n,y_n)-\barf(\xstar(y_n),y_n)\|_\Acal\leq\|x_n-\xstar(y_n)\|_\Acal\left( \lambda\sqrt{\frac{(1+q/\ell^2)}{(1+q/u^2)}}\|x_n-\xstar(y_n)\|_\Acal\right)\nonumber\\
    &= \lambda\sqrt{\frac{(1+q/\ell^2)}{(1+q/u^2)}}\|x_n-\xstar(y_n)\|_\Acal^2=2\lambda\sqrt{\frac{(1+q/\ell^2)}{(1+q/u^2)}}\Acal(x_n-\xstar(y_n))=(2-2\lambda')\Acal(x_n-\xstar(y_n)).
\end{align}
\item \textbf{Term (\ref{x1-split2}) ---} Lemma \ref{lemma:subg}(b) directly gives us
\begin{equation}\label{x1-split2-soln}
    -\langle\nabla \Acal(x_n-\xstar(y_n)),x_n-\xstar(y_n)\rangle_2\leq -2\Acal(x_n-\xstar(y_n)).
\end{equation}
\item \textbf{Term (\ref{x1-split3}) ---} Recall that 
\begin{align*}
   f(x_n,y_n,Z_n)-\barf(x_n,y_n)&=V(x_n,y_n,Z_n)-\sum_{j\in\Scal}p(j|Z_n)V(x_n,y_n,j) \nonumber\\
   &=V(x_n,y_n,Z_{n+1})-\sum_{j\in\Scal}p(j|Z_n)V(x_n,y_n,j)+V(x_n,y_n,Z_n)-V(x_n,y_n,Z_{n+1}).
\end{align*}
Also recall that ${\tilde{V}_{n+1}}=V(x_n,y_n,Z_{n+1})-\sum_{j\in\Scal}p(j|Z_n)V(x_n,y_n,j)$ is a martingale difference sequence with respect to filtration $\mathcal{F}_n$. Thus, we can write \eqref{x1-split3} as
\begin{subequations}\label{x1-split3-split}
    \begin{align}
        &\alpha_n\langle\nabla \Acal(x_n-\xstar(y_n)),f(x_n,y_n,Z_n)-\barf(x_n,y_n)+M_{n+1}\rangle_2\nonumber\\
        &=\alpha_n\langle\nabla \Acal(x_n-\xstar(y_n)),{\tilde{V}_{n+1}}+M_{n+1}\rangle_2\label{x1-split3-mart}\\
        &\;\;\;+\alpha_n\langle\nabla \Acal(x_n-\xstar(y_n)),V(x_n,y_n,Z_n)-V(x_n,y_n,Z_{n+1})\rangle_2\label{x1-split3-diff}.
    \end{align}
\end{subequations}
Let $d_n=\langle\nabla \Acal(x_n-\xstar(y_n)),V(x_n,y_n,Z_n)\rangle_2$. Then, \eqref{x1-split3-diff} can be re-written as the term of a telescoping series and additional `error terms'. 
\begin{subequations}\label{x1-split3-diff-soln}
    \begin{align}
        &\alpha_n\langle\nabla \Acal(x_n-\xstar(y_n)),V(x_n,y_n,Z_n)-V(x_n,y_n,Z_{n+1})\rangle_2\nonumber\\
        &=\alpha_n(d_n-d_{n+1})\label{tel-diff-d}\\
        &\;\;\;+\alpha_n\langle\nabla \Acal(x_{n+1}-\xstar(y_{n+1}))-\nabla \Acal(x_n-\xstar(y_n)), V(x_n,y_n,Z_{n+1})\rangle_2\label{x1-split-diff-1}\\
        &\;\;\;+\alpha_n\langle\nabla \Acal(x_{n+1}-\xstar(y_{n+1})), V(x_{n+1},y_{n+1},Z_{n+1})-V(x_n,y_n,Z_{n+1})\rangle_2\label{x1-split-diff-2}.
    \end{align}
\end{subequations}
Terms \eqref{x1-split-diff-1} and \eqref{x1-split-diff-2} are the error terms which we bound next.
\begin{itemize}
    \item \textbf{Term (\ref{x1-split-diff-1}) ---} Note that:
\begin{align*}
    &\alpha_n\langle\nabla \Acal(x_{n+1}-\xstar(y_{n+1}))-\nabla \Acal(x_n-\xstar(y_n)), V(x_n,y_n,Z_{n+1})\rangle_2\\
    \stackrel{(a)}{\leq} &\alpha_n\|\nabla \Acal(x_{n+1}-\xstar(y_{n+1}))-\nabla \Acal(x_n-\xstar(y_n))\|_2\|V(x_n,y_n,Z_{n+1})\|_2\\
    \stackrel{(b)}{\leq} &\alpha_n\|\nabla \Acal(x_{n+1}-\xstar(y_{n+1}))-\nabla \Acal(x_n-\xstar(y_n))\|_2\left(u\|V(x_n,y_n,Z_{n+1})\|\right)\\
    \stackrel{(c)}{\leq}  &\frac{\alpha_n}{q}\|x_{n+1}-\xstar(y_{n+1})-x_n+\xstar(y_n)\|_2K'u(1+\|x_n\|+\|y_n\|)\\
    \stackrel{(d)}{\leq}  &\frac{\alpha_nK'u^2}{q}(\|x_{n+1}-x_n\|+\|\xstar(y_{n+1})-\xstar(y_n)\|)(1+\|x_n\|+\|y_n\|)\\
    \stackrel{(e)}{\leq} &\frac{\alpha_nK'u^2}{q}\alpha_n\hD_1(1+\|x_n\|+\|y_n\|)^2.
\end{align*}
Here inequality (a) follows from Cauchy-Schwarz inequality and inequality (b) follows from norm equivalence ($\|\cdot\|_2\leq u\|\cdot\|$). For inequality (c), we use smoothness of $\Acal(x)$ for the first term, i.e., $\|\nabla \Acal(x_1)-\nabla \Acal(x_2)\|\leq (1/q)\|x_1-x_2\|_2$, and Lemma \ref{lem:poisson_soln-prop} for the second term. Inequality (d) follows from application of the triangle inequality and norm equivalence. Inequality (e) follows from Lemma \ref{lem:inter-bounds}\ref{lem-part:x-diff-bound}. 

Using the above bound and Lemma \ref{lem:inter-bounds}\ref{lem-part:crude-bound}, \eqref{x1-split-diff-1} is bounded as follows.
\begin{align}\label{x1-split-diff-soln-1}
    &\alpha_n\langle\nabla \Acal(x_{n+1}-\xstar(y_{n+1}))-\nabla \Acal(x_n-\xstar(y_n)), V(x_n,y_n,Z_{n+1})\rangle_2\nonumber\\
    &\leq \alpha_n^2\hA_1(1+\Acal(x_n-x^*(y_n))+\Bcal(y_n-y^*)),
\end{align}
where $\hA_1=K'u^2\hD_1\hD_2/q$.

\item \textbf{Term (\ref{x1-split-diff-2}) ---} Note that
\begin{align*}
    &\alpha_n\langle\nabla \Acal(x_{n+1}-\xstar(y_{n+1})), V(x_{n+1},y_{n+1},Z_{n+1})-V(x_n,y_n,Z_{n+1})\rangle_2\\
 \stackrel{(a)}{\leq}&\alpha_n\|x_{n+1}-\xstar(y_{n+1})\|_\Acal\|V(x_{n+1},y_{n+1},Z_{n+1})-V(x_n,y_n,Z_{n+1})\|_\Acal\\
    \stackrel{(b)}{\leq} &\frac{\alpha_n}{1+q/u^2}\|x_{n+1}-\xstar(y_{n+1})\|\|V(x_{n+1},y_{n+1},Z_{n+1})-V(x_n,y_n,Z_{n+1})\|\\
     \stackrel{(c)}{\leq}&\frac{\alpha_n}{1+q/u^2}\|x_{n+1}-\xstar(y_{n+1})\|\left(L_2\|x_{n+1}-x_n\|+L_2\|y_{n+1}-y_n\|\right).
\end{align*}
Here inequality (a) follows from Lemma \ref{lemma:subg}, inequality (b) follows from Lemma \ref{lemma:chen-A} and inequality (c) follows from Lemma \ref{lem:poisson_soln-prop}. Then using Lemma \ref{lem:inter-bounds}\ref{lem-part:gen-diff-bound} with constants $\hC_1=L_2$ and $\hC_2=L_2$, we first note that 
$$\left(L_2\|x_{n+1}-x_n\|+L_2\|y_{n+1}-y_n\|\right)\leq \alpha_n\hA_2(1+\|x_n\|+\|y_n\|),$$
where $\hA_2=\left(L_2(1+2K)+L_2\gamma_0(1+2K)\right)$. Now,
\begin{align*}
    \|x_{n+1}-\xstar(y_{n+1})\|&\leq \|x_{n+1}-x_n+\xstar(y_n)-\xstar(y_{n+1})\|+\|x_n-\xstar(y_n)\|\\
    &\leq \alpha_n\hD_1(1+\|x_n\|+\|y_n\|)+\|x_n-\xstar(y_n)\|.
\end{align*}
where the second inequality follows from \ref{lem:inter-bounds}\ref{lem-part:x-diff-bound}. Combining the above two inequalities, we get
\begin{align*}
    &\alpha_n\langle\nabla \Acal(x_{n+1}-\xstar(y_{n+1})), V(x_{n+1},y_{n+1},Z_{n+1})-V(x_n,y_n,Z_{n+1})\rangle_2\\
    \leq &\frac{\alpha_n^2\hA_2}{1+q/u^2}\left(\|x_n-x^*(y_n)\|+\alpha_n\hD_1(1+\|x_n\|+\|y_n\|)\right)(1+\|x_n\|+\|y_n\|)\\
    = &\frac{\alpha_n^2\hA_2}{1+q/u^2}\left(\|x_n-x^*(y_n)\|(1+\|x_n\|+\|y_n\|)+\alpha_n\hD_1(1+\|x_n\|+\|y_n\|)^2\right)\\
    \stackrel{(a)}{\leq} &\frac{\alpha_n^2\hA_2}{1+q/u^2}\Bigg(\frac{1}{2}\|x_n-x^*(y_n)\|^2+\frac{1}{2}(1+\|x_n\|+\|y_n\|)^2+\alpha_n\hD_1\hD_2\left(1+\Acal(x_n-x^*(y_n))+\Bcal(y_n-y^*)\right)\Bigg)\\
    \stackrel{(b)}{\leq} & \frac{\alpha_n^2\hA_2}{1+q/u^2}\Bigg(\frac{(1+q/\ell^2)}{2}\Acal(x_n-x^*(y_n))+\left(\frac{1}{2}+\alpha_n\hD_1\right)\hD_2\left(1+\Acal(x_n-x^*(y_n))+\Bcal(y_n-y^*)\right)\Bigg).
\end{align*}
Both inequalities (a) and (b) use Lemma \ref{lem:inter-bounds}\ref{lem-part:crude-bound} to bound $(1+\|x_n\|+\|y_n\|)^2$. Additionally, inequality $(a)$ uses AM-GM inequality and (b) uses Lemma \ref{lemma:chen-A}. Define $\hA_3=\frac{\hA_2}{1+q/u^2}\left(\frac{1}{2}+\alpha_0\hD_1\right)\hD_2+\frac{\hA_2(1+q/\ell^2)}{2(1+q/u^2)}$. Then
\begin{align}\label{x1-split-diff-soln-2}
    &\alpha_n\langle\nabla \Acal(x_{n+1}-\xstar(y_{n+1})), V(x_{n+1},y_{n+1},Z_{n+1})-V(x_n,y_n,Z_{n+1})\rangle_2\nonumber\\
    \leq\;&\alpha_n^2\hA_3\Bigg(1+\Acal(x_n-x^*(y_n))+\Bcal(y_n-y^*)\Bigg).
\end{align}

\end{itemize}

Combining the terms \eqref{x1-split3-mart} and \eqref{tel-diff-d} with the bounds \eqref{x1-split-diff-soln-1} and \eqref{x1-split-diff-soln-2}, we can finally bound \eqref{x1-split3}.
\begin{align}\label{x1-split3-soln}
\begin{split}
    &\alpha_n\langle\nabla \Acal(x_n-\xstar(y_n)),f(x_n,y_n,Z_n)-\barf(x_n,y_n)+M_{n+1}\rangle_2\\
    \leq&\;\alpha_n\langle\nabla \Acal(x_n-\xstar(y_n)),{\tilde{V}_{n+1}}+M_{n+1}\rangle_2+\alpha_n(d_n-d_{n+1})+\alpha_n^2(\hA_1+\hA_3)(1+\Acal(x_n-x^*(y_n))+\Bcal(y_n-y^*)).
\end{split}
\end{align}

\item \textbf{Term (\ref{x1-split4}) ---} Note that 
\begin{align*}
    \langle\nabla \Acal(x_n-\xstar(y_n)),\xstar(y_n)-\xstar(y_{n+1})\rangle_2&\stackrel{(a)}{\leq} \|x_n-\xstar(y_n)\|_\Acal\|\xstar(y_n+1)-\xstar(y_n)\|_\Acal\\
    &\stackrel{(b)}{\leq} \sqrt{\frac{1}{1+q/u^2}}\|x_n-\xstar(y_n)\|_\Acal\|\xstar(y_n+1)-\xstar(y_n)\|\\
    &\stackrel{(c)}{\leq} \frac{L_1}{\sqrt{1+q/u^2}}\|x_n-\xstar(y_n)\|_\Acal\|y_{n+1}-y_n\|\\
    &\stackrel{(d)}{\leq} \frac{L_1\beta_n(1+2K)}{\sqrt{1+q/u^2}}\left(\|x_n-\xstar(y_n)\|_\Acal(1+\|x_n\|+\|y_n\|)\right).
\end{align*}
Here inequality (a) and (b) follow from Lemmas \ref{lemma:subg} and \ref{lemma:chen-A}, respectively. Inequality (c) follows from Lipschitz nature of the map $y\mapsto \xstar(y)$ as shown in Lemma \ref{lem:xstar-Lipschitz}. The last inequality uses Lemma \ref{lem:inter-bounds}\ref{lem-part:gen-diff-bound} with $\hC_1=0$ and $\hC_2=1$. Using Young's inequality, $2ab\leq \frac{a^2}{\eta}+\eta b^2$ with $\eta=\frac{\beta_n}{\alpha_n\lambda'}$, we get
\begin{align*}
    &\langle\nabla \Acal(x_n-\xstar(y_n)),\xstar(y_n)-\xstar(y_{n+1})\rangle_2\\
    \leq\;& \beta_n\left(\frac{\alpha_n\lambda'}{2\beta_n}\|x_n-\xstar(y_n)\|^2_\Acal+\frac{\beta_n}{2\alpha_n\lambda'}\frac{L_1^2(1+2K)^2}{(1+q/u^2)}(1+\|x_n\|+\|y_n\|)^2\right)\\
    \leq\;& \alpha_n\lambda'\Acal(x_n-\xstar(y_n))+(\beta_n^2/\alpha_n)\hA_4\left(1+\Acal(x_n-x^*(y_n))+\Bcal(y_n-y^*)\right),
\end{align*}
where the last inequality follows from Lemma \ref{lem:inter-bounds}\ref{lem-part:crude-bound} and $\hA_4=L_1^2(1+2K)^2\hD_2/(2\lambda'(1+q/u^2))$.
\item \textbf{Term (\ref{x-split2}) ---} Now we bound \eqref{x-split2} using Lemma \ref{lem:inter-bounds}\ref{lem-part:x-diff-bound} and Lemma \ref{lem:inter-bounds}\ref{lem-part:crude-bound} as follows.
\begin{align*}
    \frac{1}{2q}\|x_{n+1}-x_n+\xstar(y_n)-\xstar(y_{n+1})\|_2^2&\leq \frac{\alpha_n^2\hD_1^2}{2q}(1+\|x_n\|+\|y_n\|)^2\\
    &\leq \hA_5\alpha_n^2\left(1+\Acal(x_n-x^*(y_n))+\Bcal(y_n-y^*)\right),
\end{align*}
where $\hA_5=\hD_1^2\hD_2/(2q)$.
\end{itemize}

Having bounded the four terms in \eqref{x1-split} and the term \eqref{x-split2}, we can now return to \eqref{x-split} to get a bound on $\Acal(x_{n+1}-\xstar(y_{n+1}))$. For this we define constants $\Gamma_1=\hA_1+\hA_3+\hA_4+\hA_5$. Then, combining all the bounds we get the following recursion. 
\begin{align*}
\begin{split}
    \Acal(x_{n+1}-\xstar(y_{n+1}))&\leq (1-\alpha_n\lambda')\Acal(x_{n}-\xstar(y_n))+\alpha_n\langle\nabla \Acal(x_n-\xstar(y_n)),{\tilde{V}_{n+1}}+M_{n+1}\rangle_2\\
    &+\alpha_n(d_n-d_{n+1})+\Gamma_1\left(\alpha_n^2+\frac{\beta_n^2}{\alpha_n}\right)\left(1+\Acal(x_n-x^*(y_n))+\Bcal(y_n-y^*)\right).
\end{split}   
\end{align*}
Recall that ${\tilde{V}_{n+1}}+M_{n+1}$ is a martingale difference sequence with respect to filtration $\mathcal{F}_n$. Thus, we have:
\begin{align*}
    \EE[\Acal(x_{n+1}-\xstar(y_{n+1}))|\mathcal{F}_n]&\leq (1-\alpha_n\lambda')\Acal(x_{n}-\xstar(y_n))+\EE\left[\alpha_n(d_n-d_{n+1})\mid\mathcal{F}_n\right]\\
    &+\Gamma_1\left(\alpha_n^2+\frac{\beta_n^2}{\alpha_n}\right)\left(1+\Acal(x_n-x^*(y_n))+\Bcal(y_n-y^*)\right).  
\end{align*}
Taking expectation again, and applying law of total expectation, we get
\begin{align}\label{x-exp-rec-inter}
    \EE[\Acal(x_{n+1}-\xstar(y_{n+1}))]&\leq (1-\alpha_n\lambda')\EE[\Acal(x_{n}-\xstar(y_n))]+\EE\left[\alpha_n(d_n-d_{n+1})\right]\nonumber\\
    &+\Gamma_1\left(\alpha_n^2+\frac{\beta_n^2}{\alpha_n}\right)\left(1+\EE[\Acal(x_n-x^*(y_n))+\Bcal(y_n-y^*)]\right).  
\end{align}

\subsection{\textbf{Proof for Lemma \ref{lemma:recursive}(b) - Intermediate Recursive Bound on $\EE[\Bcal(y_n-\ystar)]$}}
Let $n>0$ and $\mu'=1-\mu\sqrt{\frac{(1+q/\ell^2)}{(1+q/u^2)}}$. The constant $q$ is chosen such that $\mu'>0$.
By substituting $y_2=y_{n+1}-\ystar$ and $y_1=y_n-\ystar$ in Lemma \ref{lemma:chen-B}(a), we get   
\begin{subequations}\label{y-split}
\begin{align}
    \Bcal(y_{n+1}-\ystar)&\leq \Bcal(y_{n}-\ystar)\nonumber\\
    &\;\;+\langle\nabla \Bcal(y_n-\ystar),y_{n+1}-y_n\rangle_2\label{y-split1}\\
    &\;\;+\frac{1}{2q}\|y_{n+1}-y_n\|_2^2.\label{y-split2}
\end{align}
\end{subequations}
We further simplify the term in \eqref{y-split1} as follows.
\begin{subequations}\label{y1-split}
\begin{align}
    \langle\nabla \Bcal(y_n-\ystar),y_{n+1}-y_n\rangle_2&=\beta_n\langle\nabla \Bcal(y_n-\ystar),g(x_n,y_n,Z_n)-y_n+M'_{n+1}\rangle_2\nonumber\\
    &=\beta_n\langle\nabla \Bcal(y_n-\ystar),\barg(\xstar(y_n),y_n)-\barg(\xstar,\ystar)\rangle_2\label{y1-split1}\\
    &-\beta_n\langle\nabla \Bcal(y_n-\ystar),y_n-\ystar\rangle_2\label{y1-split2}\\
    &+\beta_n\langle\nabla \Bcal(y_n-\ystar),g(x_n,y_n,Z_n)-\barg(x_n,y_n)+M'_{n+1}\rangle_2\label{y1-split3}\\
    &+\beta_n\langle\nabla \Bcal(y_n-\ystar),\barg(x_n,y_n)-\barg(\xstar(y_n),y_n)\rangle_2\label{y1-split4}
\end{align}
\end{subequations}
The second equality follows from the definition that $\ystar$ is a fixed point for $\barg(x^*(\cdot),\cdot)$. The last two terms (\eqref{y1-split3} and \eqref{y1-split4}) are `error' terms and as in the case of iteration for $\{x_n\}$, we later show that they are negligible as compared to the other terms. We now bound the four terms in \eqref{y1-split} and the term \eqref{y-split2} as follows.
\begin{itemize}
    \item \textbf{Term \ref{y1-split1} --- } Using Lemma \ref{lemma:subg}(a), note that
$$\langle\nabla \Bcal(y_n-\ystar),\barg(\xstar(y_n),y_n)-\barg(\xstar,\ystar)\rangle_2\leq \|y_n-\ystar\|_\Bcal\|\barg(x^*(y_n),y_n)-\barg(\xstar,\ystar)\|_\Bcal.$$
Now, using Lemma \ref{lemma:chen-B}, we have
\begin{align*}
    \Bcal(\barg(x^*(y_n),y_n)-\barg(\xstar,\ystar))&\leq \frac{1}{2(1+q/u^2)}\|\barg(x^*(y_n),y_n)-\barg(\xstar,\ystar)\|^2\\
    &\stackrel{(a)}{\leq} \frac{\mu^2}{2(1+q/u^2)} \|y_n-\ystar\|^2\\
    &\leq \frac{\mu^2(1+q/\ell^2)}{(1+q/u^2)} \Bcal(y_n-\ystar).
\end{align*}
Here inequality (a) follows from the assumption that $\barg(\xstar(\cdot),\cdot)$ is contractive (Assumption \ref{assu-contractive-g}). Using the fact that $\Bcal(y)=1/2\|y\|_\Bcal^2$, we now have
\begin{align}\label{y1-split1-soln}
    \langle\nabla \Bcal(y_n-\ystar),\barg(\xstar(y_n),y_n)-\barg(\xstar,\ystar)\rangle_2&\leq\|y_n-\ystar\|_\Bcal\times \mu\sqrt{\frac{(1+q/\ell^2)}{(1+q/u^2)}}\|y_n-\ystar\|_\Bcal\nonumber\\
    &=2\mu\sqrt{\frac{(1+q/\ell^2)}{(1+q/u^2)}}\Bcal(y_n-\ystar)=(2-2\mu')\Bcal(y_n-\ystar).
\end{align}
The above inequality follows from a series of steps similar to inequality \eqref{x1-split1-soln}.
\item \textbf{Term \ref{y1-split2} --- } Lemma \ref{lemma:subg}(b) directly gives us
\begin{equation}\label{y1-split2-soln}
    -\langle\nabla \Bcal(y_n-\ystar),y_n-\ystar\rangle_2\leq -2\Bcal(y_n-\ystar).
\end{equation}

\item \textbf{Term \ref{y1-split3} --- } Note that 
\begin{align*}
   g(x_n,y_n,Z_n)-\barg(x_n,y_n)&=W(x_n,y_n,Z_n)-\sum_{j\in\Scal}p(j|Z_n)W(x_n,y_n,j) \nonumber\\
   &=W(x_n,y_n,Z_{n+1})-\sum_{j\in\Scal}p(j|Z_n)W(x_n,y_n,j)+W(x_n,y_n,Z_n)-W(x_n,y_n,Z_{n+1}).
\end{align*}
Define ${\widetilde{W}_{n+1}}=W(x_n,y_n,Z_{n+1})-\sum_{j\in\Scal}p(j|Z_n)W(x_n,y_n,j)$. Note that $\{\widetilde{W}_{n+1}\}$ is a martingale difference sequence with respect to the filtration $\mathcal{F}_n$. Thus, we can write \ref{y1-split3} as
\begin{subequations}\label{y1-split3-split}
    \begin{align}
        &\beta_n\langle\nabla \Bcal(y_n-\ystar),g(x_n,y_n,Z_n)-\barg(x_n,y_n)+M'_{n+1}\rangle_2\nonumber\\&=\beta_n\langle\nabla \Bcal(y_n-\ystar),{\widetilde{W}_{n+1}}+M'_{n+1}\rangle_2\label{y1-split3-mart}\\
        &\;\;+\beta_n\langle\nabla \Bcal(y_n-\ystar),W(x_n,y_n,Z_n)-W(x_n,y_n,Z_{n+1})\rangle_2\label{y1-split3-diff}.
    \end{align}
\end{subequations}
Let $e_n=\langle\nabla \Bcal(y_n-\ystar),W(x_n,y_n,Z_n)\rangle_2$. Then, \ref{y1-split3-diff} can be re-written as the term of a telescoping series and additional `error terms'.
\begin{subequations}\label{y1-split3-diff-soln}
    \begin{align}
        &\beta_n\langle\nabla \Bcal(y_n-\ystar),W(x_n,y_n,Z_n)-W(x_n,y_n,Z_{n+1})\rangle_2\nonumber\\
        &=\beta_n(e_n-e_{n+1})\label{tel-diff-e}\\
        &\;\;+\beta_n\langle\nabla \Bcal(y_{n+1}-\ystar)-\nabla \Bcal(y_n-\ystar), W(x_n,y_n,Z_{n+1})\rangle_2\label{y1-split-diff-1}\\
        &\;\;+\beta_n\langle\nabla \Bcal(y_{n+1}-\ystar), W(x_{n+1},y_{n+1},Z_{n+1})-W(x_n,y_n,Z_{n+1})\rangle_2\label{y1-split-diff-2}.
    \end{align}
\end{subequations}
We next bound the terms \eqref{y1-split-diff-1} and \eqref{y1-split-diff-2}.
\begin{itemize}
    \item \textbf{Term \ref{y1-split-diff-1} --- } 

Using Holder's inequality and the smoothness of $\Bcal(x)$, \ref{y1-split-diff-1} can be bounded as
\begin{align*}
    &\beta_n\langle\nabla \Bcal(y_{n+1}-\ystar)-\nabla \Bcal(y_n-\ystar), W(x_n,y_n,Z_{n+1})\rangle_2\\
    \stackrel{(a)}{\leq} &\beta_n\|\nabla \Bcal(y_{n+1}-\ystar)-\nabla \Bcal(y_n-\ystar)\|_2\|W(x_n,y_n,Z_{n+1})\|_2.\\
    \stackrel{(b)}{\leq} &\beta_n\|\nabla \Bcal(y_{n+1}-\ystar)-\nabla \Bcal(y_n-\ystar)\|_2(u\|W(x_n,y_n,Z_{n+1})\|).\\
    \stackrel{(c)}{\leq} &\frac{\beta_n}{q}\|y_{n+1}-y_n\|_2K'u(1+\|x_n\|+\|y_n\|)\\
    \stackrel{(c)}{\leq} &\frac{\beta_nu}{q}\|y_{n+1}-y_n\|K'u(1+\|x_n\|+\|y_n\|)\\
    \stackrel{(e)}{\leq} &\frac{\beta_nK'u^2}{q}\beta_n(1+2K)(1+\|x_n\|+\|y_n\|)^2.
\end{align*}
Here inequality (a) follows from Cauchy-Schwarz inequality and inequality (b) follows from norm equivalence ($\|\cdot\|_2\leq u\|\cdot\|$). For inequality (c), we use smoothness of $\Bcal(x)$ for the first term, i.e., $\|\nabla \Bcal(y_1)-\nabla \Bcal(y_2)\|\leq (1/q)\|x_1-x_2\|_2$, and Lemma \ref{lem:poisson_soln-prop} for the second term. Inequality (d) follows from application of norm equivalence and inequality (e) follows from Lemma \ref{lem:inter-bounds}\ref{lem-part:gen-diff-bound}. Define $\hB_1=\frac{(1+2K)K'u^2}{q}\hD_2\gamma_0$. Then \ref{y1-split-diff-1} is bounded as
\begin{align}\label{y1-split-diff-soln-1}
    \beta_n\langle\nabla \Bcal(y_{n+1}-\ystar)&-\nabla \Bcal(y_n-\ystar), W(x_n,y_n,Z_{n+1})\rangle_2\nonumber\\
    &\leq \beta_n^2\frac{(1+2K)K'u^2}{q}\hD_2(1+\Acal(x_n-x^*(y_n))+\Bcal(y_n-y^*))\\
    &\leq \alpha_n\beta_n\hB_1(1+\Acal(x_n-x^*(y_n))+\Bcal(y_n-y^*)).
\end{align}
\item \textbf{Term \ref{y1-split-diff-2} --- } Note that 
\begin{align*}
    &\beta_n\langle\nabla \Bcal(y_{n+1}-\ystar), W(x_{n+1},y_{n+1},Z_{n+1})-W(x_n,y_n,Z_{n+1})\rangle_2\\
    \stackrel{(a)}{\leq} &\beta_n\|y_{n+1}-\ystar\|_\Bcal\|W(x_{n+1},y_{n+1},Z_{n+1})-W(x_n,y_n,Z_{n+1})\|_h\\
    \stackrel{(b)}{\leq} &\frac{\beta_n}{1+q/u^2}\|y_{n+1}-\ystar\|\|W(x_{n+1},y_{n+1},Z_{n+1})-W(x_n,y_n,Z_{n+1})\|\\
    \stackrel{(c)}{\leq} &\frac{\beta_n}{1+q/u^2}\|y_{n+1}-\ystar\|\left(L_2\|x_{n+1}-x_n\|+L_2\|y_{n+1}-y_n\|\right).
\end{align*}
Here inequality (a) follows from Lemma \ref{lemma:subg}, inequality (b) follows from Lemma \ref{lemma:chen-B} and inequality (c) follows from Lemma \ref{lem:poisson_soln-prop}. Define $\hB_2=\left(L_2(1+2K)+L_2\gamma_0(1+2K)\right)$ (note that $\hB_2=\hA_2$). Then using Lemma \ref{lem:inter-bounds}\ref{lem-part:gen-diff-bound} with $\hC_1=L_2$ and $\hC_2=L_2$ for the second term, we get
$$\left(L_2\|x_{n+1}-x_n\|+L_2\|y_{n+1}-y_n\|\right)\leq \alpha_n\hB_2(1+\|x_n\|+\|y_n\|).$$
Next, 
$$\|y_{n+1}-\ystar\|\leq \|y_{n+1}-y_n\|+\|y_n-\ystar\|\leq \beta_n(1+2K)(1+\|x_n\|+\|y_n\|)+\|y_n-\ystar\|.$$
Hence,
\begin{align*}
    &\beta_n\langle\nabla \Bcal(y_{n+1}-\ystar), W(x_{n+1},y_{n+1},Z_{n+1})-W(x_n,y_n,Z_{n+1})\rangle_2\\
    \leq &\frac{\beta_n\alpha_n\hB_2}{1+q/u^2}\Big(\|y_n-\ystar\|+\beta_n(1+2K)(1+\|x_n\|+\|y_n\|)\Big)(1+\|x_n\|+\|y_n\|)\\
    \leq &\frac{\beta_n\alpha_n\hB_2}{1+q/u^2}\left(\frac{1+q/\ell^2}{2}\Bcal(y_n-\ystar)+\left(\frac{1}{2}+\beta_n(1+2K)\right)\hD_2(1+\Acal(x_n-\xstar(y_n))+\Bcal(y_n-\ystar))\right).
\end{align*}
Here, we skip a few steps, using the expansion from \eqref{x1-split-diff-soln-2} directly. Then,
\begin{align}\label{y1-split-diff-soln-2}
    \beta_n\langle\nabla \Bcal(y_{n+1}-\ystar)&, W(x_{n+1},y_{n+1},Z_{n+1})-W(x_n,y_n,Z_{n+1})\rangle_2\nonumber\\
    &\leq\beta_n\alpha_n\hB_3(1+\Acal(x_n-x^*(y_n))+\Bcal(y_n-y^*)).
\end{align}
where $\hB_3=\frac{\hB_2}{1+q/u^2}\left(\left(\frac{1}{2}+\beta_0(1+2K)\right)\hD_2+\frac{1+q/l^2}{2}\right)$.
\end{itemize}
Combining the terms \eqref{y1-split3-mart} and \eqref{tel-diff-e}, with the bounds \eqref{y1-split-diff-soln-1} and \eqref{y1-split-diff-soln-2}, we finally bound \eqref{y1-split3} as follows.
\begin{align}\label{y1-split3-soln}
    &\beta_n\langle\nabla \Bcal(y_n-\ystar),g(x_n,y_n,Z_n)-\barg(x_n,y_n)+M'_{n+1}\rangle_2\nonumber\\
    &\leq\beta_n\langle\nabla \Bcal(y_n-\ystar),{\widetilde{W}_{n+1}}+M'_{n+1}\rangle_2+\beta_n(e_n-e_{n+1})+\alpha_n\beta_n(\hB_1+\hB_3)(1+\Acal(x_n-x^*(y_n))+\Bcal(y_n-y^*)).
\end{align}

\item \textbf{Term \ref{y1-split4} --- } For \eqref{y1-split4}, using Lemma \ref{lemma:subg} and norm equivalence, we get
\begin{align}
    \beta_n\langle\nabla \Bcal(y_n-\ystar),\barg(x_n,y_n)-\barg(\xstar(y_n),y_n)\rangle_2\stackrel{(a)}{\leq}& \beta_n\|y_n-\ystar\|_\Bcal\|\barg(x_n,y_n)-\barg(\xstar(y_n),y_n)\|_\Bcal\nonumber\\
    \stackrel{(b)}{\leq} &\frac{\beta_n}{\sqrt{1+q/u^2}}\|y_n-\ystar\|_\Bcal\|\barg(x_n,y_n)-\barg(\xstar(y_n),y_n)\|\nonumber\\
    \stackrel{(c)}{\leq} &\frac{\beta_nL}{\sqrt{1+q/u^2}}\|y_n-\ystar\|_\Bcal\|x_n-x^*(y_n)\|\nonumber\\
    \stackrel{(d)}{\leq} &\beta_n\left(\|y_n-\ystar\|_\Bcal\frac{L\sqrt{1+q/\ell^2}}{\sqrt{1+q/u^2}}\|x_n-x^*(y_n)\|_\Acal\right)\nonumber\\
    \stackrel{(e)}{\leq} &\beta_n\left(\mu'\frac{1}{2}\|y_n-\ystar\|_\Bcal^2+\frac{L^2(1+q/\ell^2)}{2\mu'(1+q/u^2)}\|x_n-\xstar(y_n)\|_\Acal^2\right)\nonumber\\
    \stackrel{(f)}{=}&\beta_n\left(\mu'\Bcal(y_n-\ystar)+\frac{L^2(1+q/\ell^2)}{\mu'(1+q/u^2)}\Acal(x_n-\xstar(y_n))\right)\nonumber.
\end{align}
Here inequality (a) follows from Lemma \ref{lemma:subg}, inequalities (b) and (d) follow from Lemma \ref{lemma:chen-B} and (c) follows from Assumption \ref{assu-Lipschitz}. Inequality (e) follows from AM-GM inequality, $2ab\leq \frac{a^2}{\eta}+\eta b^2$ with $\eta=\mu'$ and finally equation (f) follows from Lemma \ref{lemma:chen-B}. Define $\hB_4=\frac{L^2(1+q/\ell^2)}{\mu'(1+q/u^2)}$. Hence 
\begin{equation}\label{y1-split4-soln}
    \beta_n\langle\nabla \Bcal(y_n-\ystar),\barg(x_n,y_n)-\barg(\xstar(y_n),y_n)\rangle_2\leq \beta_n\mu'\Bcal(y_n-\ystar)+\beta_n\hB_4\Acal\|x_n-\xstar(y_n)\|.
\end{equation}

\item \textbf{Term \ref{y-split2} --- } Finally we note that
\begin{align*}
    \frac{1}{2q}\|y_{n+1}-y_n\|_2^2\stackrel{(a)}{\leq} \frac{\beta_n^2(1+2K)^2}{2q}(1+\|x_n\|+\|y_n\|)^2&\stackrel{(b)}{\leq} \frac{\alpha_n\beta_n\gamma_0(1+2K)^2}{2q}(1+\Acal(x_n-x^*(y_n))+\Bcal(y_n-y^*))\\
    &= \alpha_n\beta_n\hB_5(1+\Acal(x_n-x^*(y_n))+\Bcal(y_n-y^*))
\end{align*}
where $\hB_5=\frac{\gamma_0(1+2K)^2}{2q}$. Here inequality (a) follows from Lemma \ref{lem:inter-bounds}\ref{lem-part:gen-diff-bound} and (b) follows from Lemma \ref{lem:inter-bounds}\ref{lem-part:crude-bound}.
\end{itemize}

Having bounded the four terms in \eqref{y1-split} and the term \eqref{y-split2}, we can now return to \eqref{y-split} to get a bound on $\Bcal(y_{n+1}-\ystar)$. Define constants $\Gamma_2=\hB_1+\hB_3+\hB_5$ and $\Gamma_3=\hB_4$. Then, combining all the bounds we get the following recursion.
\begin{align*}
\begin{split}
    \Bcal(y_{n+1}-\ystar)&\leq (1-\beta_n\mu')\Bcal(y_{n}-\ystar)+\beta_n\langle\nabla \Bcal(y_n-\ystar),{\widetilde{W}_{n+1}}+M'_{n+1}\rangle_2+\beta_n(e_n-e_{n+1})\\
    &+\alpha_n\beta_n\Gamma_2(1+\Acal(x_n-x^*(y_n))+\Bcal(y_n-y^*))+\beta_n\Gamma_3\Acal(x_n-\xstar(y_n))
\end{split}   
\end{align*}
Recall that ${\widetilde{W}_{n+1}}+M'_{n+1}$ is a martingale difference sequence with respect to filtration $\mathcal{F}_n$. Thus, we have:
\begin{align*}
    \EE[\Bcal(y_{n+1}-\ystar)|\mathcal{F}_n]&\leq (1-\beta_n\mu')\Bcal(y_{n}-\ystar)+\beta_n(e_n-e_{n+1})\\
    &\;\;+\alpha_n\beta_n\Gamma_2(1+\Acal(x_n-x^*(y_n))+\Bcal(y_n-y^*))+\beta_n\Gamma_3\Acal(x_n-\xstar(y_n))
\end{align*}
Taking expectation again, and applying law of total expectation, we get
\begin{align}\label{y-exp-rec-inter}
    \EE[\Bcal(y_{n+1}-\ystar)]&\leq(1-\beta_n\mu')\EE[\Bcal(y_{n}-\ystar)]+\beta_n\EE[(e_n-e_{n+1})]\nonumber\\
    &\;\;\;+\alpha_n\beta_n\Gamma_2(1+\EE[\Acal(x_n-x^*(y_n))+\Bcal(y_n-y^*)])+\beta_n\Gamma_3\EE[\Acal(x_n-\xstar(y_n))].
\end{align}
This completes our proof for Lemma \ref{lemma:recursive}.
\subsection{\textbf{Proof for Lemma \ref{lemma:bounded-expectation} --- Iterates are bounded in expectation}}
   Adding \eqref{x-exp-rec-inter} and \eqref{y-exp-rec-inter}, we get:
    \begin{align*}
        &\EE[\Acal(x_{n+1}-\xstar(y_{n+1}))+\Bcal(y_{n+1}-\ystar)]\\
        &\leq (1-\alpha_n\lambda')\EE[\Acal(x_{n}-\xstar(y_n))]+(1-\beta_n\mu')\EE[\Bcal(y_{n}-\ystar)]+\alpha_n\EE[(d_n-d_{n+1})]+\beta_n\EE[(e_n-e_{n+1})]\\
        &\;\;\;+\left(\alpha_n^2+\frac{\beta_n^2}{\alpha_n}+\alpha_n\beta_n\right)(\Gamma_1+\Gamma_2)(1+\EE[\Acal(x_n-x^*(y_n))+\Bcal(y_n-y^*)])+\beta_n\Gamma_3\EE[\EE\Acal(x_n-\xstar(y_n))].
    \end{align*}
    We first simplify the terms $\alpha_n\EE[(d_n-d_{n+1})]$ and $\beta_n\EE[(e_n-e_{n+1})]$. For this, note that 
    \begin{align*}
        \alpha_n\EE[(d_n-d_{n+1})]&\leq \alpha_{n-1}\EE[d_n]-\alpha_n\EE[d_{n+1}]+|\alpha_n-\alpha_{n-1}|\EE[|d_n|]\\
        &\leq \alpha_{n-1}\EE[d_n]-\alpha_n\EE[d_{n+1}]+\cc_2\hD_3\alpha_n^2(1+\EE[\Acal(x_n-x^*(y_n))+\Bcal(y_n-y^*)]).
    \end{align*}
    The second inequality here follows from our assumption that $|\alpha_n-\alpha_{n-1}|\leq \cc_2\alpha_n^2$ and Lemma \ref{lem:inter-bounds}\ref{lem-part:telescoping-diff-bound}.
    Similarly we have that 
    \begin{align*}
        \beta_n\EE[(e_n-e_{n+1})]&\leq  \beta_{n-1}\EE[e_n]-\beta_n\EE[e_{n+1}]+\cc_2\hD_3\beta_n^2(1+\EE[\Acal(x_n-x^*(y_n))+\Bcal(y_n-y^*)]).
    \end{align*}
    Combining these bounds, we get
    \begin{align*}
        &\EE[\Acal(x_{n+1}-\xstar(y_{n+1}))+\Bcal(y_{n+1}-\ystar)]\\
        &\leq \left(1-\alpha_n\lambda'+\beta_n\Gamma_3+(\alpha_n^2+\beta_n^2+\beta_n^2/\alpha_n+\alpha_n\beta_n)(\Gamma_1+\Gamma_2+\cc_2\hD_3)\right)\EE[\Acal(x_n-\xstar(y_n))]\\
        &\;\;\;+\left(1-\beta_n\mu'+(\alpha_n^2+\beta_n^2+\beta_n^2/\alpha_n+\alpha_n\beta_n)(\Gamma_1+\Gamma_2+\cc_2\hD_3)\right)\EE[\Bcal(y_n-\ystar)]\\
        &\;\;\;+\alpha_{n-1}\EE[d_n]-\alpha_n\EE[d_{n+1}]+\beta_{n-1}\EE[e_n]-\beta_n\EE[e_{n+1}]\\
        &\;\;\;+(\alpha_n^2+\beta_n^2+\beta_n^2/\alpha_n+\alpha_n\beta_n)(\Gamma_1+\Gamma_2+\cc_2\hD_3)
    \end{align*}

    We now define time instant $n'$ as the time instant such that for all $n>n'$, the following three properties are satisfied.
    \begin{enumerate}[label=(\roman*)]
        \item $$\beta_n\Gamma_3+(\alpha_n^2+\beta_n^2+\beta_n^2/\alpha_n+\alpha_n\beta_n)(\Gamma_1+\Gamma_2+\cc_2\hD_3)\leq \lambda'\alpha_n,$$
        \item $$(\alpha_n^2+\beta_n^2+\beta_n^2/\alpha_n+\alpha_n\beta_n)(\Gamma_1+\Gamma_2+\cc_2\hD_3)\leq \mu'\beta_n,$$
        \item $$(\alpha_n+\beta_n)\hD_3\leq 0.5.$$
    \end{enumerate}
    Using the first two properties, we get for all $n\geq n'$,
        \begin{align*}
        \EE[\Acal(x_{n+1}-\xstar(y_{n+1}))+\Bcal(y_{n+1}-\ystar)]&\leq  \EE[\Acal(x_{n}-\xstar(y_n))+\Bcal(y_{n}-\ystar)]\\
        &\;\;\;+\alpha_{n-1}\EE[d_n]-\alpha_n\EE[d_{n+1}]+\beta_{n-1}\EE[e_n]-\beta_n\EE[e_{n+1}]\\
        &\;\;\;+(\alpha_n^2+\beta_n^2+\beta_n^2/\alpha_n+\alpha_n\beta_n)(\Gamma_1+\Gamma_2+\cc_2\hD_3)
    \end{align*}
    Iterating the final recursion from $n'$ to $n$ we get
    \begin{align*}
        \EE[\Acal(x_{n+1}-\xstar(y_{n+1}))+\Bcal(y_{n+1}-\ystar)]&\leq \EE[\Acal(x_{n'}-\xstar(y_{n'}))+\Bcal(y_{n'}-\ystar)]\\
        &\;\;+\alpha_{n'-1}\EE[d_{n'}]-\alpha_{n}\EE[d_{n+1}]+\beta_{n'-1}\EE[e_{n'}]-\beta_n\EE[e_{n+1}]\\
        &\;\;+\sum_{k=n'}^n(\alpha_n^2+\beta_n^2+\beta_n^2/\alpha_n+\alpha_n\beta_n)(\Gamma_1+\Gamma_2+\cc_2\hD_3)
    \end{align*}
    Note that the summation is finite because of our assumptions on the step size. Next we use Lemma \ref{lem:inter-bounds}\ref{lem-part:telescoping-diff-bound} to get
    \begin{align*}
        -\alpha_n\EE[d_{n+1}]-\beta_n\EE[e_{n+1}]&\leq (\alpha_n+\beta_n)\hD_3(1+\EE[\Acal(x_{n+1}-\xstar(y_{n+1}))+\Bcal(y_{n+1}-\ystar)])\\
        &\leq 0.5+0.5\EE[\Acal(x_{n+1}-\xstar(y_{n+1}))+\Bcal(y_{n+1}-\ystar)].
    \end{align*}
    Here the second inequality follows from point (iii) in definition of $n'$ above. Similarly,
    \begin{align*}
        \alpha_{n-1}\EE[d_{n'}]+\beta_{n-1}\EE[e_{n'}]&\leq (\alpha_{n'}+\beta_{n'})\hD_3(1+\EE[\Acal(x_{n'}-\xstar(y_{n'}))+\Bcal(y_{n'}-\ystar)]).
    \end{align*}
     Hence, we get the following bound.
    \begin{align*}
        &0.5\EE[\Acal(x_{n+1}-\xstar(y_{n+1}))+\Bcal(y_{n+1}-\ystar)]\\
        &\leq  (1+\alpha_{n'}\hD_3+\beta_{n'}\hD_3)\EE[\Acal(x_{n'}-\xstar(y_{n'}))+\Bcal(y_{n'}-\ystar)]\\
        &\;\;\;+(\alpha_{n'}+\beta_{n'})\hD_3+0.5+\sum_{k=n'}^n(\alpha_n^2+\beta_n^2+\beta_n^2/\alpha_n+\alpha_n\beta_n)(\Gamma_1+\Gamma_2+\cc_2\hD_3).
    \end{align*}
    Note that $n'$ is a constant dependent only on the step size sequence and hence the iterates till time $n'$ can be bounded by a constant using the discrete Gronwall's inequality \cite{Borkar-book}, i.e., there exists $\hD_4$ such that $$\EE[\Acal(x_n-\xstar(y_n))+\Bcal(y_n-\ystar)]\leq \hD_4, \forall n\leq n'.$$ Then for all $n\geq 0$
    $$\EE[\Acal(x_{n}-\xstar(y_n))+\Bcal(y_{n}-\ystar)]\leq \Gamma_4,$$
    where $\Gamma_4=2\hD_4(1+\alpha_{n'}\hD_3+\beta_{n'}\hD_3)+2(\alpha_{n'}+\beta_{n'})\hD_3+1+2\sum_{i=0}^\infty(\alpha_i^2+\beta_i^2+\beta_i^2/\alpha_i+\alpha_i\beta_i)(\Gamma_1+\Gamma_2+\cc_2\hD_3)$.

\subsection{\textbf{Proof for Lemma \ref{lemma:almost-done}(a) --- Almost Sure Convergence}}
For the faster iteration updating on the faster time-scale, we study the ODE $$\dot{x}(t)=\barf(x(t),y)-x(t).$$ For a fixed $y$, we know that $f(\cdot,y)$ is a contractive mapping. Hence, using Theorem 2.1 from \cite{Borkar-book}, $\xstar(y)$ is the globally asymptotically stable equilibrium (GASE) and hence the global attractor for the above ODE. Similarly, for the slower time-scale, we study the ODE $$\dot{y}(t)=\barg(\xstar(y(t)),y(t))-y(t).$$ The function $\barg(\xstar(\cdot),\cdot)$ is contractive, and hence $\ystar$ is the GASE and the global attractor for this ODE. We have shown that the iterates are bounded in expectation and hence they are finite almost surely. Hence our iterates satisfy Assumptions A1-A7 from \cite{Karmakar-convergence}. Corollary 1 from \cite{Karmakar-convergence} then states that the iterates $(x_n,y_n)$ almost surely converge to global attractors $(\xstar,\ystar)$. 

\subsection{\textbf{Proof for Lemma \ref{lemma:almost-done}(b) --- Bound on $\EE[\Acal(x_n-\xstar(y_n))]$}}
Combining the intermediate bound from Lemma \ref{lemma:recursive} with our boundedness result (Lemma \ref{lemma:bounded-expectation}), we get
\begin{align*}
    \EE[\Acal(x_{n+1}-\xstar(y_{n+1}))]&\leq (1-\alpha_n\lambda')\EE[\Acal(x_{n}-\xstar(y_n))]+\alpha_n\EE[(d_n-d_{n+1})]+\left(\alpha_n^2+\frac{\beta_n^2}{\alpha_n}\right)\left(1+\Gamma_4\right).
\end{align*}
Iterating the above recursion from $i=0$ to $n$, we get 
\begin{align}\label{x-bound-almost}
    \EE[\Acal(x_n-\xstar(y_n))]&\leq \prod_{i=0}^{n-1}(1-\alpha_i\lambda')\EE[\Acal(x_0-\xstar(y_0))]+\sum_{i=0}^{n-1}\alpha_i\EE[d_i-d_{i+1}]\prod_{j=i+1}^{n-1}(1-\alpha_j\lambda')\nonumber\\
    &\;\;\;+ \sum_{i=0}^{n-1}\left(\alpha_i^2+\frac{\beta_i^2}{\alpha_i}\right)\prod_{j=i+1}^{n-1}(1-\alpha_j\lambda').
\end{align}

We bound the term corresponding to the telescopic series above using Lemma \ref{lem:tel_term}.
Furthermore, we use Lemma \ref{lem:inter-bounds}\ref{lem-part:telescoping-diff-bound} and Lemma \ref{lemma:bounded-expectation} to get $\EE[|d_i|]\leq \hD_3(1+\Acal(x_n-\xstar(y_n))+\Bcal(y_n-\ystar))\leq \hD_3(1+\Gamma_4).$
Let $\hD_5=\hD_3(1+\Gamma_4)(\cc_1+\cc_2)$. Combining the bounds, we have
\begin{align}\label{d_n-tight-bound}
    \sum_{i=0}^{n-1}\alpha_i\EE[d_i-d_{i+1}]\prod_{j=i+1}^{n-1}(1-\alpha_j\lambda')\leq \alpha_0\hD_3(1+\Gamma_4)\prod_{j=1}^{n-1}(1-\alpha_j\lambda')+\cc_1\alpha_{n}\hD_3(1+\Gamma_4)+\hD_5\sum_{i=1}^{n-1}\alpha_i^2\prod_{j=i+1}^{n-1}(1-\alpha_j\lambda').
\end{align}
Returning to \eqref{x-bound-almost}, we now have
\begin{align*}
    \EE[\Acal(x_n-\xstar(y_n))]&\leq \prod_{i=1}^{n-1}(1-\alpha_i\lambda')\Big(\EE[\Acal(x_0-\xstar(y_0))]+\alpha_0\hD_3(1+\Gamma_4)\Big)+\cc_1\alpha_{n}\hD_3(1+\Gamma_4)\\
    &\;\;\;+ \sum_{i=0}^{n-1}\left(\alpha_i^2+\frac{\beta_i^2}{\alpha_i}\right)\prod_{j=i+1}^{n-1}(1-\alpha_j\lambda')(1+\hD_5).
\end{align*}
Assume that $n\geq n_I$, where $n_I$ is specified in Lemma \ref{lem:step-size_bound}. Also, we substitute $\kappa=\lambda'$. Then, we can bound the above expression as
\begin{align*}
    \EE[\Acal(x_n-\xstar(y_n))]&\leq \Big(\EE[\Acal(x_0-\xstar(y_0))]+\alpha_0\hD_3(1+\Gamma_4)\Big)\exp\left[-\frac{\lambda'\alpha_0}{1-\afrak}\left((n+1)^{1-\afrak}-1\right)\right]+\cc_1\alpha_{n}\hD_3(1+\Gamma_4)\\
    &\;\;\;+ \frac{2(1+\hD_5)}{\lambda'}\left(\alpha_n+\frac{\beta_n^2}{\alpha_n^2}\right).
\end{align*}
Define $n_1\geq n_I$ such that $\exp\left[-\frac{\lambda'\alpha_0}{1-\afrak}\left((n+1)^{1-\afrak}-1\right)\right]\leq \alpha_n$ for all $n\geq n_1$. Define $\Gamma_5=\max(\Big(\EE[\Acal(x_0-\xstar(y_0))]+\alpha_0\hD_3(1+\Gamma_4)\Big), \cc_1\hD_3(1+\Gamma_4), 2(1+\hD_5)/\lambda')$. Then, we get the following final bound on $\EE[\Acal(x_n-\xstar(y_n))]$.
\begin{align*}
    \EE[\Acal(x_n-\xstar(y_n))]\leq \Gamma_5\left(\alpha_n+\frac{\beta_n^2}{\alpha_n^2}\right) ~~\forall n\geq n_1.
\end{align*}

\subsection{\textbf{Proof for Lemma \ref{lemma:almost-done}(c) --- Bound on $\EE[\Bcal(y_n-\ystar)]$}}
Combining the intermediate bound from Lemma \ref{lemma:recursive} with our boundedness result (Lemma \ref{lemma:bounded-expectation}), we get
\begin{align*}
    \EE[\Bcal(y_{n+1}-\ystar)]&\leq(1-\beta_n\mu')\EE[\Bcal(y_{n}-\ystar)]+\beta_n\EE[(e_n-e_{n+1})]+\alpha_n\beta_n\Gamma_2(1+\Gamma_4)+\beta_n\Gamma_3\EE[\Acal(x_n-\xstar(y_n))].
\end{align*}
Using Lemma \ref{lemma:almost-done}(b), $\EE[\Acal(x_n-\xstar(y_n))]\leq \Gamma_5(\alpha_n+\beta_n^2/\alpha_n^2)$. Hence,
\begin{align*}
    \EE[\Bcal(y_{n+1}-\ystar)]&\leq(1-\beta_n\mu')\EE[\Bcal(y_{n}-\ystar)]+\beta_n\EE[(e_n-e_{n+1})]+\alpha_n\beta_n\Gamma_2(1+\Gamma_4)+\Gamma_3\Gamma_5\left(\alpha_n\beta_n+\frac{\beta_n^3}{\alpha_n^2}\right).
\end{align*}
Iterating the above recursion from $i=0$ to $n$, we get
\begin{align*}
    \EE[\Bcal(y_n-\ystar)]&\leq\EE[\Bcal(y_0-\ystar)]\prod_{i=0}^{n-1}(1-\beta_i\mu')+\sum_{i=0}^{n-1}\beta_i\EE[(e_i-e_{i+1})]\prod_{j=i+1}^{n-1}(1-\beta_j\mu')\\
    &\;\;\;+\Big(\Gamma_2(1+\Gamma_4)+\Gamma_3\Gamma_5\Big)\sum_{i=0}^{n-1}\Big(\alpha_i\beta_i+\frac{\beta_i^3}{\alpha_i^2}\Big)\prod_{j=i+1}^{n-1}(1-\beta_j\mu').
\end{align*}
Using Lemma \ref{lem:tel_term} and the same technique as the bound in \eqref{d_n-tight-bound}, we get
\begin{align*}
    \sum_{i=0}^{n-1}\beta_i\EE[e_i-e_{i+1}]\prod_{j=i+1}^{n-1}(1-\beta_j\mu')\leq \beta_0\hD_3(1+\Gamma_4)\prod_{j=1}^{n-1}(1-\beta_j\mu')+\cc_1\beta_{n}\hD_3(1+\Gamma_4)+\hD_5\sum_{i=1}^{n-1}\beta_i^2\prod_{j=i+1}^{n-1}(1-\beta_j\mu').
\end{align*}
Hence,
\begin{align*}
    \EE[\Bcal(y_n-\ystar)]&\leq\Big(\EE[\Bcal(y_0-\ystar)]+\beta_0\hD_3(1+\Gamma_4)\Big)\prod_{i=1}^{n-1}(1-\beta_i\mu')+\cc_1\beta_{n}\hD_3(1+\Gamma_4)\\
    &\;\;\;+\Big(\Gamma_2(1+\Gamma_4)+\Gamma_3\Gamma_5+\hD_5\gamma_0\Big)\sum_{i=0}^{n-1}\Big(\alpha_i\beta_i+\frac{\beta_i^3}{\alpha_i^2}\Big)\prod_{j=i+1}^{n-1}(1-\beta_j\mu').
\end{align*}
Define $C_3=\max(2\afrak, 4(1-\afrak))/\mu'$ and substitute $\kappa=\mu'$. Then, we can bound the above expression as
\begin{align*}
    \EE[\Bcal(y_n-\ystar)]&\leq\Big(\EE[\Bcal(y_0-\ystar)]+\beta_0\hD_3(1+\Gamma_4)\Big)\left(\frac{1}{n+1}\right)^{\mu'\beta_0}+\cc_1\beta_{n}\hD_3(1+\Gamma_4)\\
    &\;\;\;+\Big(\Gamma_2(1+\Gamma_4)+\Gamma_3\Gamma_5+\hD_5\gamma_0\Big)\frac{2}{\mu'}\left(\alpha_n+\frac{\beta_n^2}{\alpha_n^2}\right).
\end{align*}
Note that $\beta_0\mu'\geq \max(2\afrak, 4(1-\afrak))>1$. Define $n_2$ such that $\left(\frac{1}{n+1}\right)^{\mu'\beta_0}\leq \alpha_n$ for all $n\geq n_2$. Define $\Gamma_6=\max\left(\Big(\EE[\Bcal(x_0-\xstar(y_0))]+\beta_0\hD_3(1+\Gamma_4)\Big), \cc_1\hD_3(1+\Gamma_4), 2\Big(\Gamma_2(1+\Gamma_4)+\Gamma_3\Gamma_5+\hD_5\gamma_0\Big)/\mu'\right)$. Then, we get the following final bound on $\EE[\Bcal(x_n-\xstar(y_n))]$.
\begin{align*}
    \EE[\Bcal(x_n-\xstar(y_n))]\leq \Gamma_6\left(\alpha_n+\frac{\beta_n^2}{\alpha_n^2}\right) ~~\forall n\geq n_2.
\end{align*}

\subsection{\textbf{Proof for Theorem \ref{thm:expectation-bound}}}
We have already shown the almost sure convergence result in Lemma \ref{lemma:almost-done}(a). We use the bounds from Lemma \ref{lemma:almost-done} and Lemma \ref{lemma:chen-B} to complete our proof. For all $n\geq n_1$,
\begin{align*}
    \EE\left[\|x_n-\xstar(y_n)\|^2\right]\leq 2(1+q/\ell^2)\EE[\Acal(x_n-\xstar(y_n))]\leq 2(1+q/\ell^2)\Gamma_5(\alpha_n+\beta_n^2/\alpha_n^2).
\end{align*}
Similarly, for all $n\geq n_2$,
\begin{align*}
    \EE\left[\|y_n-\ystar\|^2\right]\leq 2(1+q/\ell^2)\EE[\Bcal(y_n-\ystar)]\leq 2(1+q/\ell^2)\Gamma_6(\alpha_n+\beta_n^2/\alpha_n).
\end{align*}
Recall that $\xstar=\xstar(\ystar)$. Then, for all $n\geq \max(n_1, n_2)$,
\begin{align*}
    \EE\left[\|x_n-\xstar\|^2\right]&\leq2\EE[\|x_n-\xstar(y_n)\|^2]+2\EE[\|\xstar(y_n)-\xstar(\ystar)\|^2]\\
    &\leq 2\EE[\|x_n-\xstar(y_n)\|^2]+2L_1^2\EE[\|y_n-\ystar\|^2]\\
    &\leq 4(1+q/\ell^2)(\Gamma_5+L_1^2\Gamma_6)(\alpha_n+\beta_n^2/\alpha_n^2).
\end{align*}
This completes our proof with $C_1=4(1+q/\ell^2)(\Gamma_5+L_1^2\Gamma_6)$, $C_2=2(1+q/\ell^2)\Gamma_6$ and $n_0=\max(n_1, n_2)$.
\subsection{\textbf{Technical Lemmas Used in Proof of Theorem \ref{thm:expectation-bound}}}

\begin{lemma}\label{lem:inter-bounds}
    Suppose the setting of Theorem \ref{thm:expectation-bound} holds. Then, the following hold true.
    \begin{enumerate}[label=(\alph*)]
        \item \label{lem-part:gen-diff-bound} For constants $\hC_1,\hC_2>0$, 
        \begin{align*}
        \hC_1\|x_{n+1}-x_n\|+\hC_2\|y_{n+1}-y_n\|\leq \alpha_n(1+2K)\left(\hC_1+\hC_2\gamma_n\right)(1+\|x_n\|+\|y_n\|).
        \end{align*}
        \item \label{lem-part:x-diff-bound} Define constant $\hD_1=(1+2K)(1+L_1\gamma_0)$. Then,
        \begin{align*}
        \|x_{n+1}-x_n+\xstar(y_n)-\xstar(y_{n+1})\|\leq \alpha_n\hD_1(1+\|x_n\|+\|y_n\|).
        \end{align*} 
        \item \label{lem-part:crude-bound} Define $\hD_2=3\max\{(1+\|x^*\|+\|y^*\|)^2,(1+L_1)^2(1+q/\ell^2)\}$. Then
    \begin{align*}
        (1+\|x_n\|+\|y_n\|)^2\leq \hD_2(1+\Acal(x_n-x^*(y_n))+\Bcal(y_n-y^*)).
    \end{align*}
    \item \label{lem-part:telescoping-diff-bound}Define $\hD_3\coloneqq\frac{K'(2+\hD_2)}{2\sqrt{1+q/\ell^2}}$. Then $$|d_n|\leq \hD_3(1+\Acal(x_n-\xstar(y_n))+\Bcal(y_n-\ystar))\;\text{and}\;|e_n|\leq \hD_3(1+\Acal(x_n-\xstar(y_n))+\Bcal(y_n-\ystar)).$$
    \end{enumerate}
\end{lemma}
\begin{proof}
    \begin{enumerate}[label=(\alph*)]
        \item We have
    \begin{align*}
        \hC_1\|x_{n+1}-x_n\|+\hC_2\|y_{n+1}-y_n\|
        &\leq \Bigg(\alpha_n\hC_1\|f(x_n, y_n, Z_n)-x_n+M_{n+1}\|\\
        &\;\;\;\;\;\;\;\;\;+\beta_n\hC_2\|g(x_n,y_n,Z_n)- y_n+ M'_{n+1}\|\Bigg).
    \end{align*}
    Using the linear growth assumptions
    \begin{align*}
        \hC_1\|x_{n+1}-x_n\|+\hC_2\|y_{n+1}-y_n\|
        &\leq \alpha_n\hC_1((2K)(1+\|x_n\|+\|y_n\|)+\|x_n\|)\\
        &\;\;\;+\beta_n\hC_2(2K)(1+\|x_n\|+y_n\|)+\|y_n\|)\\
        &\leq \alpha_n\left(\hC_1(1+2K)+\hC_2\gamma_n(1+2K)\right)(1+\|x_n\|+\|y_n\|).
    \end{align*}
    \item     Using Lemma \ref{lem:xstar-Lipschitz}, we have
    \begin{align*}
        \|x_{n+1}-x_n+\xstar(y_n)-\xstar(y_{n+1})\|
        &\leq \|x_{n+1}-x_n\|+L_1\|y_n-y_{n+1}\|.
    \end{align*}
    Put $\hC_1=1$ and $\hC_2=L_1$ in the Lemma \ref{lem:inter-bounds}\ref{lem-part:gen-diff-bound} and use the fact that $\gamma_n$ is a non-increasing sequence. The claim follows.
    \item Note that
    \begin{align*}
        (1+\|x_n\|+\|y_n\|)
        &\leq (1+\|x^*(y_n)\|+\|y^*\|+\|x_n-x^*(y_n)\|+\|y_n-y^*\|).
    \end{align*}
    Using Lemma \ref{lem:xstar-Lipschitz}, we can bound $\|x^*(y_n)\|\leq L_1\|y_n-y^*\|+\|x^*\|$ to get
    \begin{align*}
        (1+\|x_n\|+\|y_n\|)^2&\leq 3((1+\|x^*\|+\|y^*\|)^2+\|x_n-x^*(y_n)\|^2+(1+L_1)^2\|y_n-y^*\|^2)\\
        &\leq 3((1+\|x^*\|+\|y^*\|)^2+(1+q/\ell^2)\Acal(x_n-x^*(y_n))\\
        &\;\;\;+(1+L_1)^2(1+q/\ell^2)\Bcal(y_n-y^*)).
    \end{align*}
    The claim follows. 
    \item Using Lemma \ref{lemma:subg}, we have:
    \begin{align*}
        |d_n|&\leq \|x_n-\xstar(y_n)\|_\Acal\|V(x_n,y_n,Z_n)\|_\Acal\\
        &\stackrel{(a)}{\leq} \frac{1}{\sqrt{1+q/\ell^2}}\|x_n-\xstar(y_n)\|_\Acal\|V(x_n,y_n,Z_n)\|\\
        &\stackrel{(b)}{\leq} \frac{K'}{\sqrt{1+q/\ell^2}}\|x_n-\xstar(y_n)\|_\Acal(1+\|x_n\|+\|y_n\|)\\
        &\stackrel{(c)}{\leq} \frac{K'}{\sqrt{1+q/\ell^2}}\left(\frac{1}{2}\|x_n-\xstar(y_n)\|_\Acal^2+\frac{1}{2}(1+\|x_n\|+\|y_n\|)^2\right)\\
        &\stackrel{(d)}{\leq} \frac{K'}{\sqrt{1+q/\ell^2}}\left(\Acal(x_n-\xstar(y_n))+\frac{1}{2}\hD_2\left(1+\Acal(x_n-x^*(y_n)+\Bcal(y_n-y^*)\right)\right)
    \end{align*}
    where we use equivalence of norms for $(a)$, Lemma \ref{lem:poisson_soln-prop} for $(b)$, Young's inequality for $(c)$ and Lemma \ref{lem:inter-bounds}\ref{lem-part:crude-bound} for $(d)$. The bound for $e_n$ follows in exactly the same manner.
    \end{enumerate}
\end{proof}

\section{Proof for Theorem \ref{thm:expectation-special}}\label{app:proof-special}
\subsubsection{\textbf{Recursive Bound on $\EE[\Acal(x_n-\xstar(y_n))]$}}
We do not repeat the complete proof for Lemma \ref{lemma:recursive}(a), and only point out the differences. The only step in which the proof differs is the bound for term \eqref{x1-split4}. Note that 
\begin{align*}
    \langle\nabla \Acal(x_n-\xstar(y_n)),\xstar(y_n)-\xstar(y_{n+1})\rangle_2&\stackrel{}{\leq} \|x_n-\xstar(y_n)\|_\Acal\|\xstar(y_n+1)-\xstar(y_n)\|_\Acal\\
    &\stackrel{}{\leq} \sqrt{\frac{1}{1+q/u^2}}\|x_n-\xstar(y_n)\|_\Acal\|\xstar(y_n+1)-\xstar(y_n)\|\\
    &\stackrel{}{\leq} \frac{L_1}{\sqrt{1+q/u^2}}\|x_n-\xstar(y_n)\|_\Acal\|y_{n+1}-y_n\|\\
    &\stackrel{}{\leq} \frac{L_1\beta_n}{\sqrt{1+q/u^2}}\left(\|x_n-\xstar(y_n)\|_\Acal\|\barg(x_n,y_n)-y_n\|\right).
\end{align*}
We bound $\|\barg(x_n,y_n)-y_n\|$ as follows
\begin{align*}
    \|\barg(x_n,y_n)-y_n\|&\leq \|\barg(x_n,y_n)-\barg(\xstar(y_n),y_n)\|+\|\barg(\xstar(y_n),y_n)-\ystar\|+\|y_n-\ystar\|\\
    &\leq L\|x_n-\xstar(y_n)\|+(1+\mu)\|y_n-\ystar\|\\
    &\leq L\|x_n-\xstar(y_n)\|+2\|y_n-\ystar\|\\
    &\leq L\sqrt{1+q/\ell^2}\|x_n-\xstar(y_n)\|_\Acal+2\sqrt{1+q/\ell^2}\|y_n-\ystar\|_\Bcal.
\end{align*}
Here the second inequality follows from the fact that $\ystar=\barg(\xstar(\ystar),\ystar)$ and that $\barg(\xstar(\cdot),\cdot)$ is $\mu$-contractive. Now,
\begin{align*}
    &\langle\nabla \Acal(x_n-\xstar(y_n)),\xstar(y_n)-\xstar(y_{n+1})\rangle_2\\
    &\leq L_1\sqrt{\frac{1+q/u^2}{1+q/\ell^2}}\Big(L\beta_n\|x_n-\xstar(y_n)\|_\Acal^2+2\beta_n\|x_n-\xstar(y_n)\|_\Acal\|y_n-\ystar\|_\Bcal\Big)\\
    &\leq L_1L\sqrt{\frac{1+q/u^2}{1+q/\ell^2}}\beta_n\|x_n-\xstar(y_n)\|_\Acal^2+\frac{16L_1^2}{\mu'}\frac{1+q/u^2}{1+q/\ell^2}\beta_n\|x_n-\xstar(y_n)\|^2_\Acal+\frac{\mu'}{4}\beta_n\|y_n-\ystar\|^2_\Bcal\\
    &= \hA_6\beta_n\Acal(x_n-\xstar(y_n))+\frac{\mu'}{2}\beta_n\Bcal(y_n-\ystar).
\end{align*}
Here $\hA_6=2LL_1\sqrt{(1+q/u^2)/(1+q/\ell^2)}+32L_1^2(1+q/u^2)/(\mu'(1+q/\ell^2))$.
This completes our bound for the term \eqref{x1-split4}. Using the same steps as in the Proof for Lemma \ref{lemma:recursive}(a), 
but replacing the bound for \eqref{x1-split4} with our new bound, we get
\begin{align}\label{x-rec-special}
    \EE[\Acal(x_{n+1}-\xstar(y_{n+1}))]&\leq (1-2\alpha_n\lambda')\EE[\Acal(x_{n}-\xstar(y_n))]+\EE\left[\alpha_n(d_n-d_{n+1})\right]\nonumber\\
    &+\hA_7\alpha_n^2\left(1+\EE[\Acal(x_n-x^*(y_n))+\Bcal(y_n-y^*)]\right)+\hA_6\beta_n\Acal(x_n-\xstar(y_n))+\frac{\mu'}{2}\beta_n\Bcal(y_n-\ystar).  
\end{align}
Here $\hA_7=\hA_1+\hA_3+\hA_5$.
\subsubsection{\textbf{Recursive Bound for $\EE[\Bcal(y_n-\ystar)]$}}
We note that due to absence of any noise, the term corresponding to \eqref{y1-split3} is zero. This gives us the following recursive bound on $\EE[\Bcal(y_n-\ystar)]$.
\begin{align}\label{y-rec-special}
    \EE[\Bcal(y_{n+1}-\ystar)]&\leq(1-\beta_n\mu')\EE[\Bcal(y_{n}-\ystar)]\nonumber\\
    &\;\;\;+\alpha_n\beta_n\Gamma_2(1+\EE[\Acal(x_n-x^*(y_n))+\Bcal(y_n-y^*)])+\beta_n\Gamma_3\EE[\Acal(x_n-\xstar(y_n))].
\end{align}
The only change from \eqref{y-exp-rec-inter} is that the term $\EE[\beta_n(e_n-e_{n+1})]$ is now zero.
\subsubsection{\textbf{Combining Bounds}}
Adding the two bounds in \eqref{x-rec-special} and \eqref{y-rec-special}, we get
\begin{align*}
    \EE[\Acal(x_{n+1}-\xstar(y_{n+1}))+\Bcal(y_{n+1}-\ystar)]
    &\leq (1-2\alpha_n\lambda'+\hA_6\beta_n+\Gamma_3\beta_n+\hA_7\alpha_n^2+\Gamma_2\alpha_n^2)\EE[\Acal(x_n-\xstar(y_n))]\\
    &\;\;+(1-\beta_n\mu'+\beta_n(\mu'/2)+\hA_7\alpha_n^2+\Gamma_2\alpha_n^2)\EE[\Bcal(y_n-\ystar)]\\
    &\;\;+\alpha_n^2(\hA_7+\Gamma_2)+\EE\left[\alpha_n(d_n-d_{n+1})\right].
\end{align*}
Here we have used our assumption that $\beta_n\leq\alpha_n$. Using the assumption that $\beta_0/\alpha_0\leq C_4$ for appropriate $C_4$, we have $(\hA_6+\Gamma_3)\beta_n\leq (\lambda'/2)\alpha_n$. Similarly, there exists $n_1$ such that for all $n\geq n_1$, $\hA_7\alpha_n^2+\Gamma_2\alpha_n^2\leq (\lambda'/2)\alpha_n$ and $\hA_7\alpha_n^2+\Gamma_2\alpha_n^2\leq (\mu'/4)\beta_n$. Then for $n>n_1$,
\begin{align*}
    \EE[\Acal(x_{n+1}-\xstar(y_{n+1}))+\Bcal(y_{n+1}-\ystar)]
    &\leq (1-\alpha_n\lambda')\EE[\Acal(x_n-\xstar(y_n))]+(1-\beta_n(\mu'/4))\EE[\Bcal(y_n-\ystar)]\\
    &\;\;+\alpha_n^2(\hA_7+\Gamma_2)+\EE\left[\alpha_n(d_n-d_{n+1})\right].
\end{align*}
Again, using our assumption that $\beta_0/\alpha_0\leq C_4$ for appropriate $C_4$, we have $(\mu'/4)\beta_n\leq \lambda'\alpha_n$. Hence for $n>n_1$,
\begin{align*}
    \EE[\Acal(x_{n+1}-\xstar(y_{n+1}))+\Bcal(y_{n+1}-\ystar)]
    &\leq (1-\beta_n(\mu'/4))\EE[\Acal(x_n-\xstar(y_n))+\Bcal(y_n-\ystar)]\\
    &\;\;+\alpha_n^2(\hA_7+\Gamma_2)+\EE\left[\alpha_n(d_n-d_{n+1})\right].
\end{align*}
\subsubsection{\textbf{Final Bound when $\alpha_n=\alpha_0/(n+1)$ and $\beta_n=\beta_0/(n+1)$}}
\begin{align*}
    \EE[\Acal(x_{n+1}-\xstar(y_{n+1}))+\Bcal(y_{n+1}-\ystar)]
    &\leq (1-\beta_n(\mu'/4))\EE[\Acal(x_n-\xstar(y_n))+\Bcal(y_n-\ystar)]\\
    &\;\;+(1/\gamma_0^2)\beta_n^2(\hA_7+\Gamma_2)+(1/\gamma_0)\EE\left[\beta_n(d_n-d_{n+1})\right].
\end{align*}
Now, iterating over to $n$,
\begin{align*}
    \EE[\Acal(x_n-\xstar(y_n))+\Bcal(y_n-\ystar)]&\leq \EE[\Acal(x_0-\xstar(y_0))+\Bcal(y_0-\ystar)]\prod_{i=0}^{n-1}(1-\beta_i\mu'/4)\\
    &\;\;\; +\frac{(\hA_7+\Gamma_2)}{\gamma_0^2}\sum_{i=0}^{n-1}\beta_i^2\prod_{j=i+1}^{n-1}(1-\beta_j\mu'/4) +\sum_{i=0}^{n-1}\beta_i\EE[(d_i-d_{i+1})]\prod_{j=i+1}^{n-1}(1-\beta_j\mu'/4)
\end{align*}

Again, using Lemma \ref{lem:tel_term} and the same technique as the bound in \eqref{d_n-tight-bound}, we get
\begin{align*}
    \sum_{i=0}^{n-1}\beta_i\EE[d_i-d_{i+1}]\prod_{j=i+1}^{n-1}\left(1-\frac{\beta_j\mu'}{4}\right)\leq \beta_0\hD_3(1+\Gamma_4)\prod_{j=1}^{n-1}\left(1-\frac{\beta_j\mu'}{4}\right)+\cc_1\beta_{n}\hD_3(1+\Gamma_4)+\hD_5\sum_{i=1}^{n-1}\beta_i^2\prod_{j=i+1}^{n-1}\left(1-\frac{\beta_j\mu'}{4}\right).
\end{align*}
Hence,
\begin{align*}
    \EE[\Acal(x_n-\xstar(y_n))+\Bcal(y_n-\ystar)]&\leq \left(\EE[\Acal(x_0-\xstar(y_0))+\Bcal(y_0-\ystar)]+\beta_0\hD_3(1+\Gamma_4)\right)\prod_{i=0}^{n-1}\left(1-\frac{\beta_j\mu'}{4}\right)\\
    &\;\;\; +\left(\frac{(\hA_7+\Gamma_2)}{\gamma_0^2}+\hD_5\right)\sum_{i=0}^{n-1}\beta_i^2\prod_{j=i+1}^{n-1}\left(1-\frac{\beta_j\mu'}{4}\right)+\cc_1\beta_{n}\hD_3(1+\Gamma_4).
\end{align*}
Define $C_5=8/\mu'$. Then, substituting $\kappa=\mu'$ in Lemma \ref{lem:step-size_bound}, we bound the above expression as 
\begin{align*}
    \EE[\Acal(x_n-\xstar(y_n))+\Bcal(y_n-\ystar)]&\leq \left(\EE[\Acal(x_0-\xstar(y_0))+\Bcal(y_0-\ystar)]+\beta_0\hD_3(1+\Gamma_4)\right)\left(\frac{1}{n+1}\right)^{\frac{\beta_0\mu'}{4}}\\
    &\;\;\; +\frac{8}{\mu'}\left(\frac{(\hA_7+\Gamma_2)}{\gamma_0^2}+\hD_5\right)\beta_n+\cc_1\beta_{n}\hD_3(1+\Gamma_4).
\end{align*}
Note that $\beta_0\mu'/4\geq 2$. Define $n_0$ such that $\left(\frac{1}{n+1}\right)^{\mu'\beta_0/4}\leq \beta_n$ for all $n\geq n_0$. Define $C_6/((1+q/\ell^2)\max(4, 4L_1^2+2))=\beta_0\max\left(\EE[\Acal(x_0-\xstar(y_0))+\Bcal(y_0-\ystar)]+\beta_0\hD_3(1+\Gamma_4), \frac{8}{\mu'}\left(\frac{(\hA_7+\Gamma_2)}{\gamma_0^2}+\hD_5\right), \cc_1\hD_3(1+\Gamma_4)\right)$. Then, it follows that
\begin{align*}
    \EE[\Acal(x_n-\xstar(y_n))+\Bcal(y_n-\ystar)]&\leq \frac{C_6}{(1+q/\ell^2)\max(4, 4L_1^2+2)(n+1)}.
\end{align*}

We use the bounds from Lemma \ref{lemma:chen-B} to complete our proof.
\begin{align*}
    \EE\left[\|x_n-\xstar(y_n)\|^2\right]&\leq 2(1+q/\ell^2)\EE[\Acal(x_n-\xstar(y_n))]\\
    \EE\left[\|y_n-\ystar\|^2\right]&\leq 2(1+q/\ell^2)\EE[\Bcal(y_n-\ystar)].
\end{align*}
Recall that $\xstar=\xstar(\ystar)$. Then, we have
\begin{align*}
    \EE\left[\|x_n-\xstar\|^2\right]&\leq2\EE[\|x_n-\xstar(y_n)\|^2]+2\EE[\|\xstar(y_n)-\xstar(\ystar)\|^2]\\
    &\leq 2\EE[\|x_n-\xstar(y_n)\|^2]+2L_1^2\EE[\|y_n-\ystar\|^2]\\
    &\leq 4(1+q/\ell^2)\left(\EE[\Acal(x_n-\xstar(y_n))+L_1^2\Bcal(y_n-\ystar)]\right).
\end{align*}
Combining all the bounds above for all $n\geq n_0$, we get
\begin{align*}
    \EE\left[\|x_n-\xstar\|^2\right]+\EE\left[\|y_n-\ystar\|^2\right]&\leq (1+q/\ell^2)\left(4\EE[\Acal(x_n-\xstar(y_n))+(4L_1^2+2)\Bcal(y_n-\ystar)]\right)\\
    &\leq (1+q/\ell^2)\max(4, 4L_1^2+2)\left(\EE[\Acal(x_n-\xstar(y_n))+\Bcal(y_n-\ystar)]\right)\\
    &\leq \frac{C_6}{n+1}.
\end{align*}

\section{Proofs for Applications}\label{app:proof-applications}
\subsection{\textbf{Proofs for SSP Q-Learning}}\label{app:SSP-proof}
\subsubsection{\textbf{Proof for Proposition \ref{prop:SSP}}}
To show the contractive nature of $\barf(\cdot, \rho)$, we first define $f_0(Q)=[f_0^{i,u}(Q)],$ where $$f_0^{i,u}(Q)=k(i,u)+\sum_{j\neq i_0}p(j|i,u)\min_{v}Q(j,v).$$  Then, as observed in \cite{tsitsiklis}, there exists some  $\lambda_0<1$ and $w$ such that 
\begin{align*}
    \|f_0(Q_1)-f_0(Q_2)\|_w&=\max_{i,u}w_{i,u}|f_0^{i,u}(Q_1)-f_0^{i,u}(Q_2)|\leq \max_{j,v} \lambda_0 w_{j,v}|Q_1(j,v)-Q_2(j,v)|=\lambda_0\|Q_1-Q_2\|_w.
\end{align*}
For any $Q_1,Q_2\in\RR^{\mathfrak{s}\mathfrak{u}}$ and $\rho\in\RR$, 
\begin{align*}
    &\|\barf(Q_1,\rho)-\barf(Q_2,\rho)\|_w\\
    &=\left\|\sum_{i\in\Scal',u\in\UU}\pi(i,u)(f(Q_1,\rho,(i,u))-f(Q_2,\rho,(i,u))\right\|_w\\
    &= \max_{i,u} w_{i,u}|(1-\pi(i,u))(Q_1(i,u)-Q_2(i,u))+\pi(i,u)(f_0^{i,u}(Q_1)-f_0^{i,u}(Q_2))|\\
    &\leq \max_{i,u}\Big( (1-\pi(i,u))w_{i,u}|(Q_1(i,u)-Q_2(i,u))|+\pi(i,u)w_{i,u}|f_0^{i,u}(Q_1)-f_0^{i,u}(Q_2))|\Big)\\
    &\leq \max_{i,u} \Big((1-\pi(i,u))\max_{j,v} w_{j,v}|(Q_1(j,v)-Q_2(j,v))|+\pi(i,u)\lambda_0\max_{j,v}w_{j,v}|Q_1(j,v)-Q_2(j,v)|\Big)\\
    &\leq (1-(1-\lambda_0)\pi_{min})\|Q_1-Q_2\|_w.
\end{align*}
Here $\pi_{min}=\min_{i,u} \pi(i,u)$. Under the assumption that the Markov chain is irreducible under any control, and the assumption that $\Phi_s(u|i)>0$ for all $i,u$, we have that $\pi_{min}>0$. And hence Proposition \ref{prop:SSP}(a) is satisfied with $\lambda=(1-(1-\lambda_0)\pi_{min})$. 

We next show the contractive nature of $\barg(Q^*(\cdot),\cdot)$. As shown in \cite{abounadi}, the map $\rho\mapsto \min_v Q^*(\rho)(i_0,v)$ is concave and piecewise linear, with finitely many linear pieces, therefore there exist constants $L'_1>0$ and $L'_2>0$ such that
\begin{align*}
    -L'_2(\rho_1-\rho_2)&\leq \min_v Q^*(\rho_1)(i_0,v)-\min_v Q^*(\rho_2)(i_0,v)\leq -L'_1(\rho_1-\rho_2),
\end{align*}
for all $\rho_1, \rho_2\in\RR$. Then for sufficiently small $\beta'$,
$$|\barg(Q^*(\rho_1),\rho_1)-\barg(Q^*(\rho_2),\rho_2)|\leq \mu |\rho_1-\rho_2|,$$
showing that Proposition \ref{prop:SSP} is satisfied with $\mu=\max\{1-\beta_0L'_1, 1-\beta_0L'_2\}<1$.

\subsubsection{\textbf{Proof for Corollary \ref{coro:SSP}}}
Proposition \ref{prop:SSP} shows that Assumptions \ref{assu-contractive-f} and \ref{assu-contractive-g} are satisfied. 

Define $\mathcal{F}_n\coloneqq \{Q_0, Z_m, m\leq n\}$. Then $\{M_{n+1}\}$ is a martingale difference sequence with respect to the filtration $\mathcal{F}_n$. Let $w_{max}=\max_{i,u} w_{i,u}$ and $w_{min}=\min_{i,u} w_{i,u}$. Now,
\begin{align*}
\|M_{n+1}\|_w\leq 2\|Q_n\|_w\;\;,\;\; \|f(Q,\rho)\|_w\leq \|k\|_w+3\|Q\|_w+w_{max}|\rho|\;\; \text{and}\;\; |g(Q,\rho)|\leq (\beta'/w_{min})\|Q\|_w+|\rho|.
\end{align*}
Also, note that,
\begin{align*}
    \|f(Q_1,\rho_1,Z)-f(Q_2,\rho_2,Z)\|_w&\leq \lambda\|Q_1-Q_2\|_w+w_{max}|\rho_1-\rho_2|,\\
    |g(Q_1,\rho_1,Z)-g(Q_2,\rho_2,Z)|&\leq (\beta'/w_{min})\|Q_1-Q_2\|+|\rho_1-\rho_2|.
\end{align*}
This implies that Assumptions \ref{assu-Lipschitz} and \ref{assu-Martingale} are satisfied with $K=\max\{2, \|k\|_w+3+w_{max}, (\beta'/w_{min})+1\}$ and $L=\max\{\lambda+w_{max}, (\beta'/w_{min})+1)\}$. Hence, application of Theorem \ref{thm:expectation-special} completes the proof for Corollary \ref{coro:SSP}.

\subsection{\textbf{Proof for Q-Learning with Polyak Averaging}}\label{app:Q-polyak-proof}
Since the behaviour policy induces an irreducible Markov chain and the state and action space are finite, Assumption \ref{assu-Markov} is satisfied. The Assumption \ref{assu-contractive-f} and \ref{assu-contractive-g} have already been shown to be satisfied with $\lambda=1-(1-\gamma)\pi_{min}$ and $\mu=0$. Recall that $M_{n+1}^{i,u}=\gamma I\{Z_n=(i,u)\}\Big(Q_n(X_{n+1},v)-\sum_{j\in\Scal'}p(j|i,u)\min_vQ(j,v)\Big)$, thus we have
\begin{align*}
    \|M_{n+1}\|_\infty\leq 2\|Q_n\|_\infty
\end{align*}
Next, note that the second iterate which corresponds to averaging is noiseless. Hence, $M_{n+1}'=0$. Thus, $K=2$ in Assumption \ref{assu-Martingale}. Using the expression for $f$ and $g$, we get
\begin{align*}
    \|f(Q_1, \bar{Q}_1, Z)-f(Q_2, \bar{Q}_2, Z)\|_\infty&=\max_{i,u}|I\{Z=(i,u)\}\left(Q_2(i,u)-Q_1(i,u)+\gamma\sum_{j\in S'}p(j|i,u)\left(\min_v Q_1(j,v)-\min_vQ_2(j,v)\right)\right)\\
    &\quad+Q_1(i,u)-Q_2(i,u)\\
    &\leq \max_{i,u}\Bigg(Q_2(i,u)-Q_1(i,u)+\gamma\sum_{j\in S'}p(j|i,u)\left(\min_v Q_1(j,v)-\min_vQ_2(j,v)\right)\\
    &\quad+Q_1(i,u)-Q_2(i,u)\Bigg)\\
    &\leq 3\|Q_1-Q_2\|_\infty.
\end{align*}
and $\|g(Q_1, \bar{Q}_1)-g(Q_2, \bar{Q}_2)\|_\infty\leq \|Q_1-Q_2\|_\infty$. Moreover, 
\begin{align*}
    \|f(Q, \bar{Q}, Z)\|_\infty+\|g(Q, \bar{Q})\|_\infty&\leq \max_{i,u}|k(i,u)|+3\|Q\|_\infty+\|Q\|_\infty\\
    &\leq \max(\max_{i,u}|k(i,u)|, 3)\left(1+\|Q\|_\infty\right)
\end{align*}
Thus, Assumption \ref{assu-Lipschitz} is satisfied with $L=3$ and $K=\max(\max_{i,u}|k(i,u)|, 3)$. With all the assumptions satisfied and constants determined, we can apply Theorem \ref{thm:expectation-special} to Q-Learning with Polyak-Rupper Averaging.

\subsection{\textbf{Proofs for GNEP Learning in Strongly Monotone Games}}\label{app:GNEP-proof}
Throughout the following proofs, $\|A\|$ for a matrix $\|A\|$ denotes the matrix norm of $A$ induced by the $\ell_2$ norm.
\subsubsection{\textbf{Proof for Proposition \ref{prop:GNEP}}}
\begin{enumerate}[label=\alph*)]
    \item Define $\alpha'=\lambda_0/\ell^2$. Note that $-F(\xx)$ is $\lambda_0$ strongly monotone and $\ell$-Lipschitz (Assumption \ref{assu-game}). Then by \cite[Lemma D.1]{Chandak-nonexp}, function $(\xx+\alpha'F(\xx))$ is $\lambda$-contractive where $\lambda=\sqrt{1-\lambda_0^2/\ell^2}$. Then note that $f(x_n,y_n)$ is $\lambda$-contractive in $x$.
    \item Now using \cite[Lemma 5.6 a)]{Chandak-nonexp}, $-A\xx^*(y)$ is $\mu_0=\lambda_0/(\ell\|(AA^T)^{-1}A\|)^2$ strongly monotone. Also note that $-A\xx^*(y)$ is $\ell'$-Lipschitz where $\ell_0=\|A\|L'/(1-\lambda)$. Define $\beta'=\mu_0/\ell_0^2$. Then using \cite[Lemma D.1]{Chandak-nonexp}, we have that function $(y+\beta'A\xx^*(y))$ is $\mu$-contractive where $\mu=\sqrt{1-\mu_0^2/\ell_0^2}$.
    \item Since $f(\xx,y)$ and $g(\xx^*(y),y)$ are contractive in $\xx$ and $y$, respectively, there exist unique fixed points such that $f(\xx^*(\ystar),\ystar)=\xx^*(\ystar)$ and $g(\xx^*(\ystar),\ystar)=\ystar$.
\end{enumerate}

\subsubsection{\textbf{Proof for Corollary \ref{coro:GNEP}}}
We first note that $f(x,y)$ is $L'$-Lipschitz where $L'=1+\alpha'\ell+\alpha'\|B\|$ and $g(x,y)$ is $L''$-Lipschitz where $L''=1+\beta'\|A\|$. Define $L=\max\{L',L''\}$. Note that $\|f(x,y)\|_2\leq \|x\|_2+\alpha'\ell\|x\|_2+\alpha'\|F(0)\|_2+\alpha'\|B\|y$ and $\|g(x,y)\|_2\leq \|y\|+\beta'\|A\|\|x\|_2+\beta'\|b\|_2$. Define $K=\max\{1+\alpha'\ell+\beta'\|A\|,1+\alpha'\|B\|,\alpha'\|F(0)\|_2+\beta'\|b\|_2\}$. Assumption \ref{assu-Lipschitz} is then satisfied with parameters $L$ and $K$. Having defined all required constants for Assumptions \ref{assu-contractive-f}-\ref{assu-Lipschitz}, application of Theorem \ref{thm:expectation-bound} completes the proof for Corollary \ref{coro:GNEP}.

\section{Auxiliary Lemmas}

\begin{lemma}\label{lem:tel_term}
    Let assumption \ref{assu-stepsize} hold for the step sizes $\alpha_n$ and $\beta_n$. Then, we have the following bounds
    \begin{enumerate}
        \item 
        \begin{align*}
            \sum_{i=0}^{n-1}\alpha_i\EE[d_i-d_{i+1}]\prod_{j=i+1}^{n-1}(1-\alpha_j\lambda')\leq  \alpha_0\EE[|d_0|]\prod_{j=1}^{n-1}(1-\alpha_j\lambda')+\cc_1\alpha_{n}\EE[|d_n|]+\sum_{i=0}^{n-1}\EE[|d_i|](\cc_1+\cc_2)\alpha_i^2\prod_{j=i+1}^{n-1}(1-\alpha_j\lambda').
        \end{align*}
        \item 
        \begin{align*}
            \sum_{i=0}^{n-1}\beta_i\EE[e_i-e_{i+1}]\prod_{j=i+1}^{n-1}(1-\beta_j\mu')\leq  \beta_0\EE[|e_0|]\prod_{j=1}^{n-1}(1-\beta_j\mu')+\cc_1\beta_{n}\EE[|e_n|]+\sum_{i=0}^{n-1}\EE[|e_i|](\cc_1+\cc_2')\beta_i^2\prod_{j=i+1}^{n-1}(1-\beta_j\mu').
        \end{align*}
        \item 
        \begin{align*}
            \sum_{i=0}^{n-1}\beta_i\EE[d_i-d_{i+1}]\prod_{j=i+1}^{n-1}\left(1-\frac{\beta_j\mu'}{4}\right)&\leq \beta_0\EE[|d_0|]\prod_{j=1}^{n-1}\left(1-\frac{\beta_j\mu'}{4}\right)+\cc_1\beta_{n}\EE[|d_n|]\\
            &\quad+\sum_{i=0}^{n-1}\EE[|d_i|](\cc_1+\cc_2')\beta_i^2\prod_{j=i+1}^{n-1}\left(1-\frac{\beta_j\mu'}{4}\right).
        \end{align*}
    \end{enumerate}
    \begin{proof}
        Since the proof for all the claims are identical, we only present the proof for the first bound. For this, note that
        \begin{align*}
            \sum_{i=0}^{n-1}\alpha_i\EE[d_i-d_{i+1}]\prod_{j=i+1}^{n-1}(1-\alpha_j\lambda')&= \alpha_0\EE[d_0]\prod_{j=1}^{n-1}(1-\alpha_j\lambda')-\alpha_{n-1}\EE[d_n]\\
            &\;\;\;+\sum_{i=1}^{n-1}\left(-\alpha_{i-1}\EE[d_{i}]\prod_{j=i}^{n-1}(1-\alpha_j\lambda')+\alpha_i\EE[d_i]\prod_{j=i+1}^{n-1}(1-\alpha_j\lambda')\right).
        \end{align*}
        For the last term here, we note that
        \begin{align*}
            -\alpha_{i-1}\EE[d_{i}]\prod_{j=i}^{n-1}(1-\alpha_j\lambda')+\alpha_i\EE[d_i]\prod_{j=i+1}^{n-1}(1-\alpha_j\lambda')&=(\alpha_i\EE[d_i]-\alpha_{i-1}\EE[d_i](1-\alpha_i\lambda'))\prod_{j=i+1}^{n-1}(1-\alpha_j\lambda')\\
            &=
            \EE[d_i]\Big((\alpha_i-\alpha_{i-1})+\alpha_i\alpha_{i-1}\lambda'\Big)\prod_{j=i+1}^{n-1}(1-\alpha_j\lambda')\\
            &\leq \EE[|d_i|](\cc_1+\cc_2)\alpha_i^2\prod_{j=i+1}^{n-1}(1-\alpha_j\lambda').
        \end{align*}
        Combining the above relation and using $\EE[X]\leq \EE[|X|]$ for any scalar random variable, we get
        \begin{align*}
            \sum_{i=0}^{n-1}\alpha_i\EE[d_i-d_{i+1}]\prod_{j=i+1}^{n-1}(1-\alpha_j\lambda')&\leq \alpha_0\EE[|d_0|]\prod_{j=1}^{n-1}(1-\alpha_j\lambda')+\cc_1\alpha_{n}\EE[|d_n|]+\EE[|d_i|](\cc_1+\cc_2)\alpha_i^2\prod_{j=i+1}^{n-1}(1-\alpha_j\lambda').
        \end{align*}
    \end{proof}
\end{lemma}

\begin{lemma}\label{lem:step-size_bound}
Let assumption \ref{assu-stepsize} hold for the step sizes $\alpha_n$ and $\beta_n$. Then, we have the following bounds.
    \begin{enumerate}
        \item Define $n_I=\max(2[\afrak/(\lambda'\alpha_0)]^{1/(1-\afrak)}, 2[(2-2\afrak)/(\lambda'\beta_0^2)]^{1/(2\afrak-1)})$. Then
        \begin{align*}
            \sum_{i=0}^{n-1}\left(\alpha_i^2+\frac{\beta_i^2}{\alpha_i}\right)\prod_{j=i+1}^{n-1}(1-\alpha_j\lambda')\leq \frac{2}{\lambda'}\left(\alpha_n+\frac{\beta_n^2}{\alpha_n^2}\right)~~\forall n\geq n_I.
        \end{align*}
        \item Assume that $\beta_0\geq \max(2\afrak/\mu', 4(1-\afrak)/\mu')$. Then
        \begin{align*}
            \sum_{i=0}^{n-1}\Big(\alpha_i\beta_i+\frac{\beta_i^3}{\alpha_i^2}\Big)\prod_{j=i+1}^{n-1}(1-\beta_j\mu')\leq \frac{2}{\mu'}\left(\alpha_n+\frac{\beta_n^2}{\alpha_n^2}\right) ~~\forall n\geq 0.
        \end{align*}
        \item Assume that $\beta_0\geq 8/\mu'$. Then
        \begin{align*}
            \sum_{i=0}^{n-1}\beta_i^2\prod_{j=i+1}^{n-1}\left(1-\frac{\beta_j\mu'}{4}\right)\leq \frac{8\beta_n}{\mu}.
        \end{align*}
        \item For any constant $\kappa>0$
        \begin{align*}
            \prod_{i=1}^{n-1}(1-\alpha_i\kappa)&\leq \exp\left[-\frac{\kappa\alpha_0}{1-\afrak}\left((n+1)^{1-\afrak}-1\right)\right]~~\forall n\geq 0.\\
            \prod_{i=1}^{n-1}(1-\beta_i\kappa)&\leq \left(\frac{1}{n+1}\right)^{\kappa\beta_0}~~\forall n\geq 0.
        \end{align*}
    \end{enumerate}
\end{lemma}
\begin{proof}
    \begin{enumerate}
        \item We will use a similar proof technique as in \cite{chenfinite}. Define the sequence $\{u_n\}_{n\geq 0}$ as follows
        \begin{align*}
            u_0=0;~~~~u_{n+1}=(1-\alpha_n\lambda')u_n+\left(\alpha_n^2+\frac{\beta_n^2}{\alpha_n}\right).
        \end{align*}
        Then, it is easy to verify that $u_n=\sum_{i=0}^{n-1}\left(\alpha_i^2+\frac{\beta_i^2}{\alpha_i}\right)\prod_{j=i+1}^{n-1}(1-\alpha_j\lambda')$. We will use induction to show that for $n\geq n_I$, $u_n\leq\frac{2}{\lambda'}\left(\alpha_n+\frac{\beta_n^2}{\alpha_n^2}\right)$. Since $u_0=0$, the base case is easily verified. Let us assume that the bound holds for some $n$. Then for $n+1$, we have
        \begin{align*}
            \frac{2}{\lambda'}\left(\alpha_{n+1}+\frac{\beta_{n+1}^2}{\alpha_{n+1}^2}\right)-u_{n+1}&\geq \frac{2}{\lambda'}\left(\alpha_{n+1}+\frac{\beta_{n+1}^2}{\alpha_{n+1}^2}\right)-(1-\alpha_n\lambda')\frac{2}{\lambda'}\left(\alpha_n+\frac{\beta_n^2}{\alpha_n^2}\right)-\left(\alpha_n^2+\frac{\beta_n^2}{\alpha_n}\right)\\
            &=\frac{2}{\lambda'}\alpha_n^2\left(\frac{\alpha_{n+1}-\alpha_n}{\alpha_n^2}+\frac{\lambda'}{2}\right)+\frac{2}{\lambda'}\frac{\beta_n^2}{\alpha_n}\left(\frac{\alpha_n}{\beta_n^2}\left(\frac{\beta_{n+1}^2}{\alpha_{n+1}^2}-\frac{\beta_n^2}{\alpha_n^2}\right)+\frac{\lambda'}{2}\right)
        \end{align*}
        Note that for any $\xi\in[0,1]$, we have
        \begin{align*}
            \frac{1}{(n+2)^\xi}-\frac{1}{(n+1)^\xi}=\frac{1}{(n+1)^\xi}\left(\left[\left(1+\frac{1}{n+1}\right)^{n+1}\right]^{-\frac{\xi}{n+1}}-1\right)&\geq \frac{1}{(n+1)^\xi}\left(\exp{^{-\frac{\xi}{n+1}}}-1\right)\\
            &\geq -\frac{1}{(n+1)^\xi}\frac{\xi}{n+1}
        \end{align*}
        where we used $\left(1+\frac{1}{x}\right)^x<e$ and $e^x\geq 1+x$ for all $x\in \mathbb{R}$. Using this relation for $\xi=\afrak$ and $\xi=2-2\afrak$ for the first and second term, respectively, we get
        \begin{align*}
            \frac{2}{\lambda'}\left(\alpha_{n+1}+\frac{\beta_{n+1}^2}{\alpha_{n+1}^2}\right)-u_{n+1}&\geq \frac{2}{\lambda'}\alpha_n^2\left(-\frac{\afrak}{\alpha_0(n+1)^{1-\afrak}}+\frac{\lambda'}{2}\right)+\frac{2}{\lambda'}\frac{\beta_n^2}{\alpha_n}\left(-\frac{(2-2\afrak)\alpha_0}{\beta_0^2(n+1)^{2\afrak-1}}+\frac{\lambda'}{2}\right)\\
            &\geq 0
        \end{align*}
        where the last inequality follows from $n\geq n_I=\max(2[\afrak/(\lambda'\alpha_0)]^{1/(1-\afrak)}, 2[(2-2\afrak)/(\lambda'\beta_0^2)]^{1/(2\afrak-1)})$. Thus, by induction $u_n\leq \frac{2}{\lambda'}\left(\alpha_n+\frac{\beta_n^2}{\alpha_n^2}\right)$.
        \item Using the same set of arguments as in the previous part, we define the sequence $\{u_n\}_{n\geq 0}$ as follows
        \begin{align*}
            u_0=0;~~~~u_{n+1}=(1-\beta_n\mu')u_n+\left(\alpha_n\beta_n+\frac{\beta_n^3}{\alpha_n^2}\right).
        \end{align*}
        For this part, we want to show that $u_n\leq\frac{2}{\mu'}\left(\alpha_n+\frac{\beta_n^2}{\alpha_n^2}\right)$. Then following the steps from previous part, we have
        \begin{align*}
            \frac{2}{\mu'}\left(\alpha_{n+1}+\frac{\beta_{n+1}^2}{\alpha_{n+1}^2}\right)-u_{n+1}&\geq \frac{2}{\mu'}\alpha_n\beta_n\left(-\frac{\afrak}{\beta_0}+\frac{\mu'}{2}\right)+\frac{2}{\mu'}\frac{\beta_n^3}{\alpha_n^2}\left(-\frac{(2-2\afrak)}{\beta_0}+\frac{\mu'}{2}\right).
        \end{align*}
        Since $\beta_0\geq \max(2\afrak/\mu', 4(1-\afrak)/\mu')$, the first and second term are greater than zero, we have
        \begin{align*}
            \frac{2}{\mu'}\left(\alpha_{n+1}+\frac{\beta_{n+1}^2}{\alpha_{n+1}^2}\right)-u_{n+1}&\geq 0.
        \end{align*} 
        This completes the induction.
        \item Repeating the arguments from part 1, we define the sequence $\{u_n\}_{n\geq 0}$ as follows
        \begin{align*}
            u_0=0;~~~~u_{n+1}=\left(1-\frac{\beta_n\mu'}{4}\right)u_n+\beta_n^2.
        \end{align*}
        For this part, we want to show that $u_n\leq\frac{8}{\mu'}\beta_n$. Then, following the steps from part 1, we have
        \begin{align*}
            \frac{8}{\mu'}\beta_{n+1}-u_{n+1}&\geq \frac{8}{\mu'}\beta_n^2\left(-\frac{1}{\beta_0}+\frac{\mu'}{8}\right).
        \end{align*}
        Since $\beta_0\geq 8/\mu'$, the r.h.s. is greater than zero, we have
        \begin{align*}
            \frac{8}{\mu'}\beta_{n+1}-u_{n+1}&\geq 0.
        \end{align*}
        This completes the induction.
        \item The proof for this part can be found in the arguments for the proof of Corollary 2.1.2 in \cite{chenfinite}. We do not repeat them here for brevity.
    \end{enumerate}
\end{proof}

\end{document}